\documentclass[acmsmall]{acmart}

\usepackage[english]{babel}

\usepackage{amsmath,amsfonts}
\usepackage{booktabs}       
\usepackage{nicefrac}       
\usepackage{microtype}      
\usepackage{xcolor}
\usepackage[T1]{fontenc}    
\usepackage{amsthm}
\usepackage{textcomp}
\usepackage{stfloats}
\usepackage{url}
\usepackage{verbatim}
\usepackage{graphicx}
\usepackage{bm}
\usepackage{cleveref}
\usepackage{enumitem}
\usepackage{mathtools}
\theoremstyle{plain}
\theoremstyle{definition}

\usepackage{array}
\usepackage{lipsum}
\usepackage{multirow}
\usepackage{makecell}
\usepackage{xcolor}
\usepackage{tikz}
\usepackage{ulem}
\usepackage{etoolbox}
\newcommand*{\circled}[1]{\lower.7ex\hbox{\tikz\draw (0pt, 0pt)%
    circle (.5em) node {\makebox[1em][c]{\small #1}};}}
\robustify{\circled}

\usepackage{afterpage}
\usepackage{makecell}
\usepackage{tabulary}
\usepackage{xcolor}
\usepackage{colortbl}
\usepackage{prettyref}
\usepackage{framed}
\usepackage{booktabs}       

\newcommand{\savehyperref}[2]{\texorpdfstring{\hyperref[#1]{#2}}{#2}}
\newrefformat{eq}{\savehyperref{#1}{Eq. \textup{(\ref*{#1})}}}
\newrefformat{eqn}{\savehyperref{#1}{Eq.~\textup{(\ref*{#1})}}}
\newrefformat{lem}{\savehyperref{#1}{Lemma~\ref*{#1}}}
\newrefformat{assump}{\savehyperref{#1}{Assumption~\ref*{#1}}}
\newrefformat{defi}{\savehyperref{#1}{Definition~\ref*{#1}}}\newrefformat{tab}{\savehyperref{#1}{Table~\ref*{#1}}}
\newrefformat{lemma}{\savehyperref{#1}{Lemma~\ref*{#1}}}
\newrefformat{line}{\savehyperref{#1}{Line~\ref*{#1}}}
\newrefformat{thm}{\savehyperref{#1}{Theorem~\ref*{#1}}}
\newrefformat{corr}{\savehyperref{#1}{Corollary~\ref*{#1}}}
\newrefformat{cor}{\savehyperref{#1}{Corollary~\ref*{#1}}}
\newrefformat{sec}{\savehyperref{#1}{Section~\ref*{#1}}}
\newrefformat{app}{\savehyperref{#1}{Appendix~\ref*{#1}}}
\newrefformat{ex}{\savehyperref{#1}{Example~\ref*{#1}}}
\newrefformat{fig}{\savehyperref{#1}{Figure~\ref*{#1}}}
\newrefformat{alg}{\savehyperref{#1}{Algorithm~\ref*{#1}}}
\newrefformat{rem}{\savehyperref{#1}{Remark~\ref*{#1}}}
\newrefformat{conj}{\savehyperref{#1}{Conjecture~\ref*{#1}}}
\newrefformat{prop}{\savehyperref{#1}{Proposition~\ref*{#1}}}
\newrefformat{proto}{\savehyperref{#1}{Protocol~\ref*{#1}}}
\newrefformat{prob}{\savehyperref{#1}{Problem~\ref*{#1}}}
\newrefformat{claim}{\savehyperref{#1}{Claim~\ref*{#1}}}
\newrefformat{que}{\savehyperref{#1}{Question~\ref*{#1}}}
\newrefformat{op}{\savehyperref{#1}{Open Problem~\ref*{#1}}}
\newrefformat{fn}{\savehyperref{#1}{Footnote~\ref*{#1}}}
\newrefformat{property}{\savehyperref{#1}{Property~\ref*{#1}}}

\usepackage{xspace}
\usepackage{amsmath}
\usepackage{amsthm}
\usepackage{bbm}
\usepackage{algorithm}
\usepackage[noend]{algpseudocode}
\usepackage{bm}
\usepackage{makecell}
\usepackage{tabulary}
\usepackage{subcaption}

\DeclareMathOperator*{\argmin}{argmin} 

\newcommand{\calD}{\mathcal{D}}
\newcommand{\calL}{\mathcal{L}}
\newcommand{\calN}{\mathcal{N}}

\newcommand{\calR}{\mathcal{R}}
\newcommand{\calS}{\mathcal{S}}

\newcommand{\calRO}{\mathcal{O}}

\newcommand{\EE}{\mathbb{E}}

\newcommand{\calX}{\mathcal{X}}

\newcommand{\gray}[1]{{\color{gray}#1}}

\newtheorem{theorem}{Theorem}
\newtheorem{lemma}{Lemma}
\newtheorem{definition}{Definition}
\newtheorem{prop}{Proposition}
\newtheorem{remark}{Remark}

\theoremstyle{definition}

\usepackage[most]{tcolorbox}

\begin{document}

\title{Optimizing Privacy, Utility, and Efficiency in Constrained Multi-Objective Federated Learning}

\author{Yan Kang$^*$}
\email{yangkang@webank.com}
\affiliation{%
  \institution{WeBank}
  \country{China}}

\author{Hanlin Gu$^*$}
\email{allengu@webank.com}
\affiliation{%
  \institution{WeBank}
  \country{China}
}

\author{Xingxing Tang}
\email{xtangav@cse.ust.hk}
\affiliation{%
  \institution{Hong Kong University of Science and Technology}
  \country{China}
}

\author{Yuanqin He}
\email{yuanqinhe@webank.com}
\affiliation{%
  \institution{WeBank}
  \country{China}}

\author{Yuzhu Zhang}
\email{yz-zhang22@mails.tsinghua.edu.cn}
\affiliation{%
  \institution{Shenzhen International Graduate School, Tsinghua University}
  \country{China}
}

\author{Jinnan He}
\email{hjn22@mails.tsinghua.edu.cn}
\affiliation{%
  \institution{Shenzhen International Graduate School, Tsinghua University}
  \country{China}
}

\author{Yuxing Han}
\email{yuxinghan@sz.tsinghua.edu.cn}
\affiliation{%
  \institution{Shenzhen International Graduate School, Tsinghua University}
  \country{China}
}

\author{Lixin Fan}
\email{lixinfan@webank.com}
\affiliation{%
  \institution{WeBank}
  \country{China}
}

\author{Kai Chen}
\email{kaichen@cse.ust.hk}
\affiliation{%
  \institution{Hong Kong University of Science and Technology}
  \country{China}
}

\author{Qiang Yang$\dagger$}
\email{qyang@cse.ust.hk}
\affiliation{%
  \institution{WeBank and Hong Kong University of Science and Technology}
  \country{China}
}

\renewcommand{\shortauthors}{Kang, et al.}

\def\thefootnote{*}\footnotetext{These authors contributed equally to this work}\def\thefootnote{\arabic{footnote}}
\def\thefootnote{$\dagger$}\footnotetext{Corresponding author}\def\thefootnote{\arabic{footnote}}

\begin{abstract}
Conventionally, federated learning aims to optimize a single objective, typically the utility. However, for a federated learning system to be trustworthy, it needs to simultaneously satisfy multiple/many objectives, such as maximizing model performance, minimizing privacy leakage and training cost, and being robust to malicious attacks. Multi-Objective Optimization (MOO) aiming to optimize multiple conflicting objectives simultaneously is quite suitable for solving the optimization problem of Trustworthy Federated Learning (TFL). In this paper, we unify MOO and TFL by formulating the problem of constrained multi-objective federated learning (CMOFL). Under this formulation, existing MOO algorithms can be adapted to TFL straightforwardly. Different from existing CMOFL works focusing on utility, efficiency, fairness, and robustness, we consider optimizing privacy leakage along with utility loss and training cost, the three primary objectives of a TFL system. We develop two improved CMOFL algorithms based on NSGA-II and PSL, respectively, for effectively and efficiently finding Pareto optimal solutions, and provide theoretical analysis on their convergence. We design measurements of privacy leakage, utility loss, and training cost for three privacy protection mechanisms: Randomization, BatchCrypt (An efficient homomorphic encryption), and Sparsification. Empirical experiments conducted under the three protection mechanisms demonstrate the effectiveness of our proposed algorithms.

\end{abstract}



\maketitle

\section{Introduction}




Due to the increasingly stricter legal and regulatory constraints (e.g., GDPR\footnote{GDPR is applicable as of May 25th, 2018 in all European member states to harmonize data privacy laws across Europe. \url{https://gdpr.eu/}} and HIPAA\footnote{HIPAA is a federal law of the USA created in 1996. It required the creation of national standards to protect sensitive patient health information from being disclosed without the patient’s consent or knowledge}) enforced on user privacy, private data from different users or organizations cannot be directly merged to train machine learning models. In recent years, federated learning (FL) has emerged as a practical privacy-preserving solution to tackle data silo issues without sharing users' private data. Initially, FL~\cite{mcmahan2017communication} was proposed to build models using data from millions of mobile devices. It then extended to enterprise settings where the number of participants might be much smaller, but privacy concerns are paramount~\cite{yang2019federated}. For now, FL has been widely applied to various domains such as finance~\cite{kang2021privacy,long2020federated}, healthcare~\cite{rieke2020future,dinh2022health} and advertisement~\cite{tan2020rec,yang2020federated}.

Conventionally, an FL system~\cite{McMahan2016Comm} aims to optimize a single objective, typically the utility (i.e., model performance), while treating other objectives as secondaries. However, if an FL system aims to optimize only a single objective, it likely fails to meet the requirements of other crucial objectives. For example, when a differential privacy-protected FedAvg algorithm~\cite{Wei2020feddp} wants to maximize the model performance solely, it has to reduce the noise added to protect data privacy, thereby increasing the risk of leaking privacy. For an FL system to be trusted by people (e.g., FL participants, users, and regulators), it must simultaneously fulfill multiple objectives, such as maximizing model performance, minimizing privacy leakage and training costs, and being robust to malicious attacks. We call an FL system a \textit{trustworthy federated learning system} if it optimizes the tradeoff among at least privacy, utility, and efficiency.

\begin{figure*}[!t]
\centering
\includegraphics[width=0.98\linewidth]{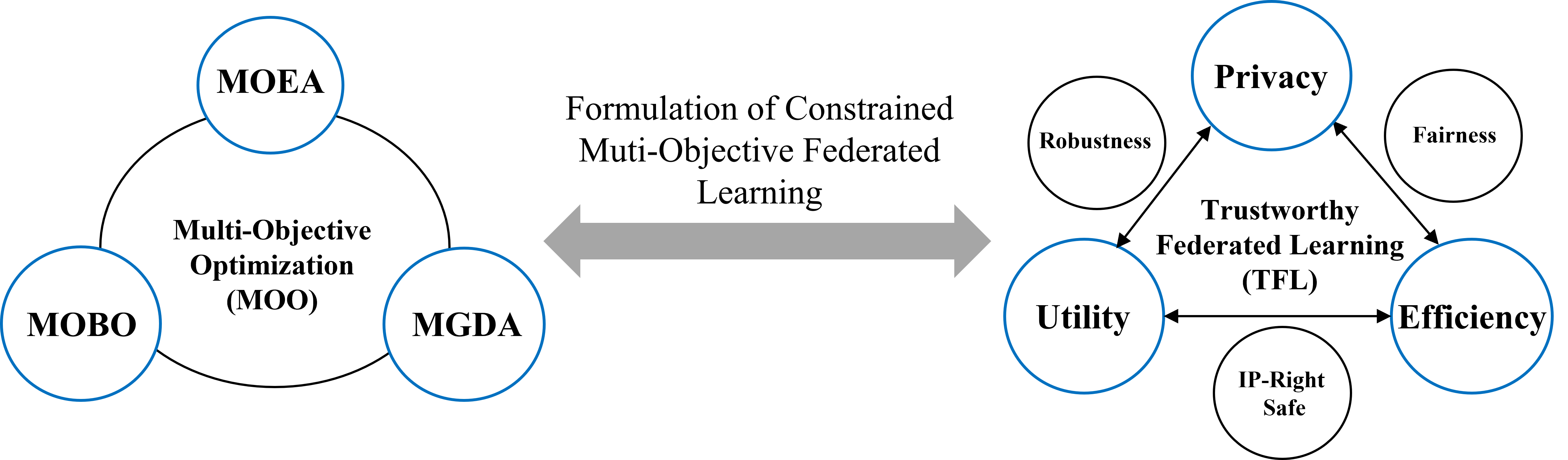}
\small
\vspace{-1em}
\caption{Formulating the problem of constrained multi-objective federated learning to unify the multi-objective optimization and trustworthy federated learning (see Sec. \ref{sec:formulation} for detail).} 
\label{fig:moo_tfl}
\end{figure*}

In this paper, we propose to apply multi-objective optimization (MOO) to optimize multiple objectives of trustworthy federated learning (TFL). To this end, we unify MOO and TFL by formulating the problem of constrained multi-objective federated learning (CMOFL), as illustrated in Figure \ref{fig:moo_tfl}. This unified formulation allows us to adapt MOO algorithms to find Pareto optimal solutions for TFL. While TFL can involve an arbitrary number of objectives, we focus on privacy, utility, and efficiency, which are primary pillars of trustworthy federated learning. The benefits of finding Pareto optimal solutions and front that optimize the three objectives include:
\begin{itemize}
    \item Pareto optimal solutions with different optimal trade-offs among conflicting objectives can support participants' flexible requirements.
    \item Solving the Pareto front for practical problems helps define the applicable boundaries of privacy protection mechanisms. For example, when DP is applied, the Pareto front can tell us that the privacy leakage is deteriorating noticeably as the protection strength reduces for maintaining utility (see Figure \ref{fig:pareto_front} (b) and (e)). Therefore, DP is unsuitable for applications in which utility and privacy are critical.
    \item  Pareto front can be considered as a comprehensive evaluation of the privacy-preserving capability of a protection mechanism concerning privacy, utility, and efficiency, and thus can guide the standard-setting for privacy levels. 
    
\end{itemize}

Existing efforts \cite{Hangyu2020nsgafl,zhong2022optimizing,Morell2022flcop,Chai04,2022arXiv220902428Y} have been devoted to leveraging multi-objective optimization algorithms to find Pareto optimal solutions that optimize utility, efficiency, robustness, and fairness. In this paper, we develop two improved constrained multi-objective federated learning algorithms based on NSGA-II~\cite{NSGA-II} and PSL~\cite{linParetoSetLearning2022}, respectively, to effectively and efficiently find Pareto optimal solutions that minimize privacy leakage, utility loss, and training cost.

In sum, our main contributions mainly include:

\begin{itemize}
    \item We formulate the problem of constrained multi-objective federated learning (CMOFL), which unifies multi-objective optimization and trustworthy federated learning. The formulation involves an average-case CMOFL and a worst-case CMOFL (see Def. \ref{def:constra-multi-objective-FL} and Def. \ref{def:worst-constra-multi-objective-FL}). The former aims to minimize the objectives of the FL system, while the latter aims to minimize the objectives of individual participants.
    \item We aim to minimize privacy leakage, along with utility loss and training cost, which are the three primary concerns of trustworthy federated learning. This is one of the first attempts in CMOFL to consider privacy an objective to optimize. 
    \item We develop two improved constrained multi-objective federated learning (CMOFL) algorithms based on NSGA-II~\cite{NSGA-II} and PSL~\cite{linParetoSetLearning2022}, respectively (Sec. \ref{sec:cmofl-algos}), which can find better Pareto optimal solutions than their baseline counterparts. The two CMOFL algorithms leverage the regret function to penalize solutions that violate privacy leakage or training cost constraint during optimization. We provide the convergence analysis of the two improved algorithms and show that the two algorithms can converge to Pareto optimal solutions with sufficient generations or population size. 
    \item We design measurements of privacy leakage, utility loss, and training cost for Randomization, BatchCrypt (an efficient version of homomorphic encryption)~\cite{zhang2020batchcrypt}, and Sparsification, respectively. Empirical experiments on each of the three protection mechanisms demonstrate the efficacy of our algorithms.
\end{itemize}
\section{Related Work}

We briefly review works related to our study of constrained multi-objective federated learning. 

\subsection{Privacy Attacking and Protection Mechanisms in FL}\label{sec:related:attack}


Federated learning protects data privacy by keeping participants' private data locally and sharing only model information (parameters and gradients). However, recent research on \textit{Deep Leakage from Gradients} (DLG)~\cite{zhu2019dlg} and its follow-up works~\cite{geiping2020inverting,zhao2020idlg,yin2021see} demonstrated that adversaries could reconstruct private data by solving a gradient-match optimization problem based on the shared model information. 

The literature has proposed various protection mechanisms to prevent private data from being inferred by adversaries. The widely used ones are \textit{Homomorphic Encryption (HE)}~\cite{gentry2009fully,batchCryp,kang2021privacy}, \textit{Differential Privacy}~\cite{geyer2017differentially,truex2020ldp,abadi2016deep}, \textit{Secret Sharing}~\cite{SecShare-Adi79,SecShare-Blakley79,bonawitz2017practical} and \textit{Sparsification} \cite{nori2021fast}. Recently, several works integrated split learning into FL~\cite{gu2021federated,Thapa2022SplitFed,wu2021fedcg} to protect privacy by splitting the local model of each party into private and public models and sharing only the public one with the server. Essentially, these works are special cases of sparsification. Another line of privacy protection  works~\cite{zhang2018mixup,huang2020instahide} is instance encoding, which transforms original training data to encoded ones such that the machine learning models are trained using encoded data, thereby preventing private data from being inferred by adversaries.

\subsection{Multi-Objective Optimization} 
Multi-objective optimization (MOO) optimizes multiple conflicting objectives simultaneously and outputs a set of Pareto optimal solutions instead of a single one~\cite{Deb2014}. A Pareto optimal solution represents a trade-off between different objectives, each of which can not be further enhanced without deteriorating others. MOO approaches can be gradient-free and gradient-based. 

Gradient-free MOO approaches do not require knowledge about the problem and thus can solve black-box optimization problems. A large volume of existing MOO works focuses on designing gradient-free MOO methods. These methods are typically based on evolutionary algorithms and Bayesian optimization. Multi-Objective Evolutionary Algorithms (MOEA), such as NSGA-II~\cite{NSGA-II}, NSGA-III~\cite{Deb2014nsga3}, SPEA2~\cite{zitzler2001spea2,kim2004spea2+} and MOEA/D~\cite{zhang2007moead}, converge fast and can find diverse Pareto solutions. However, MOEA is computationally expensive, especially when the black-box objective functions are expensive to evaluate. 

To reduce the cost of expensive objective functions, a Multi-Objective Bayesian Optimization (MOBO)~\cite{marco2002,biswas2021,yang2019945,daulton2022mobohd} algorithm approximates a surrogate model to each black-box objective function and leverages an acquisition method to search for Pareto optimal solutions. MOBO can find Pareto optimal solutions with a small objective function evaluation budget. However, the quality of the Pareto optimal solutions depends on the performance of surrogate models. 

With gradient information, gradient-based MOO approaches~\cite{desideri2012mgda,mahapatra2020mtl,liu2021stochastic} can be applied to large-scale optimization problems, such as learning neural networks. These methods mainly leverage multiple gradient descent algorithm (MSGA)~\cite{desideri2012mgda,desideri2012mgdav} to find Pareto optimal solution. The core idea of MSGA is to find the gradient with the direction that can decrease all objectives simultaneously. 


\subsection{Multi-Objective Federated Learning} \label{sec:rel_mofl}

In a multi-objective federated learning problem (MOFL), participants aim to optimize multiple competing objectives (e.g., privacy leakage, utility loss, learning cost, and fairness). MOFL typically leverages multi-objective optimization approaches to find Pareto solutions for these objectives. Existing research on MOFL mainly has two lines of direction. Table \ref{tab:mofl_survey} summarizes these works. 
\begin{table}[!h]
\footnotesize
\centering
\caption{Existing multi-objective federated learning works.}
\vspace{-0.5em}
\begin{tabular}{c||c|c|c}
\hline
\multicolumn{1}{c||}{Work} & \multicolumn{1}{c|}{Objective Scope} & \multicolumn{1}{c|}{Specific Objectives}  & \multicolumn{1}{c}{Type of Algorithm} \\
\hline
\hline
\cite{2023arXiv230109357J,NEURIPS2021_db8e1af0,2022arXiv220109917M,2022arXiv220902428Y} & Worst-case & Fairness, Utility & Gradient-based\\
\hline
\multirow{1}{*}{\cite{Hangyu2020nsgafl,zhong2022optimizing,Morell2022flcop,Chai04,2022arXiv220902428Y}} & \multirow{1}{*}{Average-case} & \multirow{1}{*}{\shortstack{Efficiency, Fairness, Robustness, Utility}} & \multirow{1}{*}{Gradient-free}\\
\hline
\end{tabular}
    \label{tab:mofl_survey}
\end{table}

The first line of MOFL works~\cite{2023arXiv230109357J,NEURIPS2021_db8e1af0,Zeou2022fedmgda,2022arXiv220109917M} treats each client's local model utility as an objective and typically aims to achieve fairness among clients by optimizing model parameters. ~\cite{Zeou2022fedmgda,NEURIPS2021_db8e1af0} leveraged Multi-Gradient Descent Algorithm (MGDA) to find a common descent direction for all objectives. \cite{2022arXiv220109917M} achieves fairness by enabling the server to assign each client a score based on a validation dataset. This gives the server the control to reward cooperative parties and punish uncooperative parties. \cite{2023arXiv230109357J} proposed adaptive federated Adam as the server optimizer to accelerate fair federated learning with bias. 

The second line of MOFL works \cite{Hangyu2020nsgafl,zhong2022optimizing,Morell2022flcop,Chai04,2022arXiv220902428Y} aims to find Pareto solutions that minimize system objectives, such as communication cost and resource consumed for performing FL training and test error of the global model. A solution represents variables that impact these system objectives. These variables can be the model structure (e.g., number of layers, number of neurons in each layer, and size of filters), training hyperparameters (e.g., learning rate, batch size, and local epoch), and communication-reduction schemes (e.g., ratio or portion of model parameters to be shared with the server and quantization bits). \cite{Hangyu2020nsgafl,zhong2022optimizing,Chai04} leveraged NSGA-II, NSGA-III and MOEA/D, respectively, to minimize communication cost and global model test error by optimizing the model structure and training hyperparameters. \cite{Morell2022flcop} also tries to minimize communication cost and global model test error but by optimizing the communication-reduction scheme.

These two lines of research apply multi-objective optimization algorithms to minimize objectives decided by federated learning parties (individually or as a whole). Another relevant but different research direction~\cite{XU2021107532,Xu2021,2022arXiv221008295L} aims to leverage federated learning to train a global surrogate model and an acquisition function using data dispersed among multiple parties. 


 

\section{Background}
We review the concepts of multi-objective optimization and trustworthy federated learning.

\subsection{Multi-Objective Optimization}

A multi-objective optimization problem is typically formulated as follows:
\begin{align}
\begin{array}{l@{\quad}l@{}l@{\quad}l}
\min\limits_{{x} \in \mathcal{X}} ( f_1({x}), f_2({x}), \ldots , f_m({x})),
\end{array}
\end{align} 
where $x$ is a solution in the decision space $\mathcal{X}$ and $\{f_i\}_{i=1}^m$ are the $m$ objectives. 

No single solution can simultaneously optimize all objectives for a non-trivial multi-objective optimization problem with conflicting objectives. Therefore, decision-makers are often interested in Pareto optimal solutions for supporting their flexible requirements. We have the following definitions of Pareto dominance, Pareto optimal solutions, Pareto set, and Pareto front.



\begin{definition}[Pareto Dominance]\label{def:pareto_dom}
Let $x_a, x_b \in \mathcal{X}$,  $x_a$ is said to dominate $x_b$, denoted as $x_a \prec x_b$, if and only if $f_i(x_a) \leq f_i(x_b), \forall i \in \{1,\ldots,m\}$ and $f_j(x_a) <  f_j(x_b), \exists j \in \{1,\ldots,m\}$.
\end{definition}

\begin{definition}[Pareto Optimal Solution]\label{def:pareto_sol}
A solution $x^{*} \in \mathcal{X}$ is called a Pareto optimal solution if there does not exist a solution $\hat{x} \in \mathcal{X}$ such that $\hat{x} \prec x^*$.
\end{definition}

Intuitively, a Pareto optimal solution represents a trade-off between conflicting objectives, each of which can not be further enhanced without negatively affecting others. All Pareto optimal solutions form a Pareto set, and their corresponding objectives form the Pareto front. The Pareto set and Pareto front are formally defined as follows.

\begin{definition}[Pareto Set and Front]\label{def:pareto_set_front}
The set of all Pareto optimal solutions is called the Pareto set, and its image in the objective space is called the Pareto front.
\end{definition}

In order to compare Pareto fronts achieved by different MOFL algorithms, we need to quantify the quality of a Pareto front. To this end, we adopt the hypervolume (HV) indicator \cite{Eckart2004hv} as the metric to evaluate Pareto fronts. Definition \ref{def:hypervolume} formally defines the hypervolume.



\begin{definition}[Hypervolume Indicator]\label{def:hypervolume}
Let $z = \{z_1,\cdots, z_m\}$ be a reference point that is an upper bound of the objectives $Y = \{y_1,\ldots, y_m\}$, such that $y_{i} \leq z_i$, $\forall i \in [m]$. the hypervolume indicator $\text{HV}_z(Y)$ measures the region between $Y$ and $z$ and is formulated as:
\begin{equation}
    \text{HV}_z(Y) = \Lambda \left( \left \{ q \in \mathbb{R}^m \big| q \in \prod_{i=1}^{m}[y_i, z_i]  \right \}\right) 
\end{equation}
where $\Lambda(\cdot)$ refers to the Lebesgue measure.
\end{definition}

Intuitively, the $\text{HV}(Y)$ can be described as the size of the space covered by $Y$. The larger the HV value, the better the $Y$.

\subsection{Trustworthy Federated Learning}
The motivation for adopting federated learning (FL) is that FL enables multiple participants to train a machine learning model that performs much better than the model trained by a single participant by leveraging the private data of all participants but without sharing them. As the research on FL goes deeper and broader, FL has encountered many new challenges. For example, FL algorithms without any protection mechanism applied are vulnerable to gradient-based privacy attacks~\cite{Lyu2022privacy}, such as the deep leakage from gradients (DLG). To address these challenges and build a trustworthy FL system, various aspects beyond utility should be considered. 
The literature has offered several proposals for what aspects a trustworthy Artificial Intelligence (AI) system~\cite{Scott2020tai,li2023tai,liu2023tai,Chen2023sai} in general and a trustworthy FL system~\cite{zhang2023survey,zhang2023trading,yang2023fedipr,jsan12010013} in particular should satisfy.
We call an FL system a \textit{trustworthy federated learning system} if it optimizes the tradeoff among at least the following three factors: privacy, utility, and efficiency.
\begin{itemize}
    \item \textbf{Privacy}. the private data of participants should be protected to prevent them from being inferred by adversaries. Privacy is of paramount importance to trustworthy federated learning. 
    \item \textbf{Utility}. The performance of the FL model should be maximized on the test data. A protection mechanism may lead to utility loss. Thus, there may have a tradeoff between privacy leakage and utility loss.
    \item \textbf{Efficiency}. The FL models should be trained or make inferences within an acceptable or controllable computation and communication cost. A protection mechanism may result in decreased efficiency. Thus, there may have a tradeoff between privacy leakage and efficiency.
\end{itemize}
\vspace{-0.5em}
\begin{remark}
We consider efficiency as an essential aspect of trustworthy FL because 
privacy protection mechanisms are often time-consuming operations, 
which makes them impractical to apply if efficiency is not optimized.
\end{remark}

\vspace{-0.4em}
Other aspects that are also crucial to trustworthy federated learning include:
\begin{itemize}
    \item \textbf{Robustness}. The FL system should tolerate extreme conditions, such as being attacked by malicious adversaries.   
    \item \textbf{Fairness}. Fairness largely impacts participants' willingness to join federated learning. We consider fairness from two perspectives: (1) the performance of FL model(s) should be maximized on each participant's test data; (2) the payoff should be fairly distributed to each participant according to their contributions to the federation~\cite{Lim2020hier,Yu2020fairinc}. 
    \item \textbf{IP-right security}. Private models' intellectual property (IP) should be protected, traced, and audited to prevent valuable model assets from being stolen or breached.
    \item \textbf{Explainability}. The decisions made by FL models should be understood by both technical FL participants and non-technical users and regulators.
\end{itemize}

\section{A Constrained Multi--Objective Federated Learning Framework}\label{sec:framework}

In this section, we formulate the constrained multi-objective federated learning problem, and define privacy leakage, learning cost, and utility loss, the three objectives we consider in this work. 
\begin{table*}[!htp]
\footnotesize
  \centering
  \setlength{\belowcaptionskip}{15pt}
  \caption{Table of Notation}
  \label{table:notation}
  \vspace{-1em}
    \begin{tabular}{cc}
    \toprule
    Notation & Meaning\cr
    \midrule\
    $X$ and $Y$ & Solutions and objective values of multi-objective optimization\\[0.1em]
    $\epsilon_{p}$, $\epsilon_u$, and $\epsilon_c$ & Privacy leakage, utility loss, and training cost  \\[0.1em]
    $W^{\calRO}_k$ & Unprotected model parameters of client $k$ \\[0.3em]
    $W^{\calD}_k$ & Protected model parameters of client $k$ \\[0.3em]
    $W_{\text{fed}}^{\calD}$ & Protected global model parameters \\[0.3em]
 $P^{\calRO}_k$ & Distribution of $W^{\calRO}_k$  \\[0.3em]
 $P^{\calD}_k$ & Distribution of $W^{\calD}_k$  \\
    \bottomrule
    \end{tabular}
\end{table*}

\subsection{General Setup}\label{sec:setup}
\noindent\textbf{FL setting and Treat Model.} We focus on \textit{horizontal federated learning} (HFL), which involves $K$ participating parties that each holds a private dataset $D_k, k \in [K]$. The server is \textit{semi-honest}, and it may launch \textit{privacy attacks} on exchanged information to infer participants' private data. To mitigate privacy leakage, each participant applies a protection mechanism to the model information that will be shared with the server.
Table \ref{table:notation} summarizes frequently used notations in this paper.

The training procedure of trustworthy federated learning with $K$ clients involves at least following four steps (also see Figure \ref{fig:mofl} (a)):
\begin{enumerate}[label=\circled{\arabic*}]
\item With the global model downloaded from the server, each client $k$ trains its local model using its private data set $D_k$, and obtains the local model $W_k^{\calRO}$.

\item In order to prevent semi-honest adversaries from inferring other clients' private information $D_k$ based on shared model $W_k^{\calRO}$, each client $k$ adopts a protection mechanism $\mathcal{M}$ (e.g., DP and HE) to convert model $W_k^{\calRO} $ to protected model $W_k^{\calD}$, and sends $W_k^{\calD}$ to the server. 

\item The server aggregates $W_k^{\calD}, k=1,\cdots,K$ to generate a new global model $W_{\text{fed}}^{\calD}$.

\item Each client $k$ downloads the global model $W_{\text{fed}}^{\calD}$ and trains its local model based on $W_{\text{fed}}^{\calD}$. (note that if the protection mechanism is HE, each client need to decrypt $W_k^{\calD}$ before local training).
\end{enumerate}

The steps \circled{1}-\circled{4} iterate until the algorithm reaches the termination condition.\\
\vspace{-1.0em}

\subsection{Problem Formulation}\label{sec:formulation}
Conventionally, the $K$ participants of FL aim to collaboratively minimize a \textit{single} objective, typically the test error of global model $W_{\text{fed}}$. This federated optimization problem is formulated as~\cite{McMahan2016Comm}:
    \begin{equation}\label{equ:fl}
        \min_{W_{\text{fed}}}f(W_{\text{fed}}) \qquad\text{where}\qquad f(W_{\text{fed}})=\frac{n_k}{n} \sum_{k = 1}^K F_{k}(W_{\text{fed}}),
    \end{equation} 
where $n_k$ denotes the size of the dataset $D_k$, $n=\sum_{k=1}^K n_k$, and $F_{k}(W_{\text{fed}}) = \frac{1}{|D_k|}\sum_{i\in D_k} \ell(W_{\text{fed}}, D_{k,i})$ is the loss of predictions made by the model parameter $W_{\text{fed}}$ on dataset $D_k$, in which $D_{k,i}$ denotes the $i$-th data-label pair of client $k$.                                          


However, the formulation in Eq (\ref{equ:fl}) does not satisfy the demand of \textit{trustworthy} federated learning (TFL), which may involve \textit{multiple/many} objectives, such as utility loss, privacy leakage, training cost, and robustness. These objectives are typically conflicting with each other and cannot be minimized to optimal simultaneously. In addition, TFL participants typically have constraints on objectives for the final solutions to be practical or feasible in real-world federated learning applications. For example, participants may require the leaked privacy to be within an acceptable threshold or the training time within a reasonable amount of time. Multi-objective optimization is used to optimize multiple objectives at the same time while considering constraints. Thus, it is perfectly suitable to solve the optimization problem of trustworthy federated learning. We unify MOO and TFL, and formulate the constrained multi-objective federated learning problem as follows.

\begin{definition}[Average-case Constrained Multi-Objective Federated Learning] \label{def:constra-multi-objective-FL}
We formalize the average-case Constrained Multi-Objective Federated Learning (CMOFL) problem as follows:
\begin{equation}
    \begin{split}
        &\min\limits_{x \in \mathcal{X}} ( f_1(x), f_2(x), \ldots , f_m(x) ) \text{, where } f_i(x) = \sum_{k=1}^Kp_{i,k}F_{i,k}(x)\text{ for $1\leq i\leq m$}, \\
        & \text{  subject to  } \,\, f_j(x) \leq \phi_j, \forall j\in \{1,\cdots,m\}
    \end{split}
\end{equation}
where $x \in \mathbb{R}^d$ is a solution in the decision space $\mathcal{X}$, $\{f_i\}_{i=1}^m$ are the $m$ objectives to minimize, $F_{i,k}$ is the local objective of client $k$ corresponds to the $i$th objective $f_i$,  $p_{i,k}$ is the coefficient of $F_{i,k}$ satisfying $\sum_{k=1}^Kp_{i,k}=1$, and $\phi_i$ is the upper constraint of $f_i$.
\end{definition}

\vspace{-0.6em}
\begin{remark}
A solution $x$ in this paper refers to a set of \textit{hyperparameters}, e.g., the learning rate, batch size, and protection strength parameters. Each hyperparameter set corresponds to specific privacy leakage, utility loss, and training cost values. We look for Pareto optimal solutions of hyperparameters that simultaneously minimize privacy leakage, utility loss, and training cost.
\end{remark}
\vspace{-0.3em}


Definition \ref{def:constra-multi-objective-FL} considers objectives that each is a weighted average of participants' local objectives. The average-case objectives can be considered as the objectives of the whole FL system. In Definition \ref{def:worst-constra-multi-objective-FL}, we also define CMOFL that optimizes the objectives of individual participants.


\begin{definition}[Worst-case Constrained Multi-Objective Federated Learning] \label{def:worst-constra-multi-objective-FL}
We formalize the worst-case Constrained Multi-Objective Federated Learning problem as follows:
\begin{equation}
    \begin{split}
        &\min\limits_{x \in \mathcal{X}} ( f_1(x), f_2(x), \ldots , f_m(x) )  \text{, where } f_i(x) = \max_{1\leq k \leq K}F_{i,k}(x)\text{ for $1\leq i\leq m$},\\
       & \text{ subject to  } \,\, f_j(x) \leq \phi_j, \forall j\in \{1,\cdots,m\}
        \end{split}
\end{equation}
where $x \in \mathbb{R}^d$ is a solution in the decision space $\mathcal{X}$, $\{f_i\}_{i=1}^m$ are the $m$ objectives to minimize and the $i$th objective $f_i, i\in[m]$ is the maximum (i.e., worst-case) of participants' $i$th local objectives $\{F_{i,k}\}_{k=1}^K$, and $\phi_i$ is the upper bound of $f_i$. 
\end{definition}

In this work, we focus on the average-case constrained multi-objective federated learning, and we are particularly interested in minimizing privacy leakage, utility loss, and training cost, the three primary concerns of trustworthy federated learning.


\subsection{Privacy Leakage, Utility Loss and Learning Cost}


To mitigate privacy leakage, each participant applies a protection mechanism to transform original model parameters $W_k^\calRO$ to a \textit{distorted} (i.e., protected) ones $W_k^\calD$ and sends $W_k^\calD$ to the server for further training. This implies that a protection mechanism impacts not only privacy leakage but also utility loss and training cost. In this section, we provide protection-mechanism-agnostic definitions of privacy leakage $\epsilon_p$, utility loss $\epsilon_u$, and training cost $\epsilon_c$. Based on these definitions, we will provide specific measurements of $\epsilon_p$, $\epsilon_u$, and  $\epsilon_c$ for Randomization, BatchCrypt, and Sparsification in Sec. \ref{sec:exp_setup} for experiments.


\noindent\textbf{Privacy Leakage.}\label{def:bayesian_privacy}
Following \cite{du2012privacy,rassouli2019optimal}, the distortion (i.e., protection) extent is defined as the distance between the distribution $P_k^\calRO$ of original model parameters $W_k^\calRO$ (i.e., $W_k^\calRO \sim P_k^\calRO$) and the distribution $P_k^\calD$ of protected model parameters $W_k^\calD$ (i.e., $W_k^\calD \sim P_k^\calD$). In this work, we leverage the distortion extent to formulate the privacy leakage:
\begin{align}\label{eq: def_of_pl}
\epsilon_p = \frac{1}{K}\sum_{k=1}^K \epsilon_{p,k} \text{,  where } \epsilon_{p,k} = 1 - \text{TV}(P_k^\calRO||P_k^\calD),
\end{align}
where $\text{TV}(\cdot||\cdot)$ denotes the Total Variation distance between two distributions. A larger distortion applied to the original model parameters leads to larger $\text{TV}(P_k^\calRO||P_k^\calD)$, thereby less privacy leakage.  
\begin{remark}
We employ the TV distance (in the range of [0, 1]) to measure the distortion extent following \cite{duchi2013local}. Zhang et al. \cite{zhang2022no} defined a Bayesian privacy leakage by measuring the information on the private data that a semi-honest adversary may infer when observing protected model information. They demonstrated that the proposed $\epsilon_p$ in Eq. (\ref{eq: def_of_pl}) serves as a lower bound for Bayesian privacy leakage (Theorem 4.1 in \cite{zhang2022no}).
\end{remark}

The privacy leakage defined in Eq. (\ref{eq: def_of_pl}) shows that a larger distortion leads to less privacy leakage. On the other hand, a larger distortion may result in larger utility loss and training cost. We define the utility loss and training cost as follows.

\noindent\textbf{Utility Loss.}
The utility loss $\epsilon_u$ of a federated learning system is defined as follows:
    \begin{equation}\label{def:utility}
        \epsilon_u = \text{U}({W_{\text{fed}}^{\calRO}}) - \text{U}(W_{\text{fed}}^{\calD}),
    \end{equation}
where $\text{U}(W_{\text{fed}}^{\calD})$ and $\text{U}(W_{\text{fed}}^{\calRO})$ measure the utility of protected global model $W_{\text{fed}}^{\calD}$ and unprotected global model $W_{\text{fed}}^{\calRO}$, respectively.


\noindent\textbf{Training Cost}
The training cost $\epsilon_c$ of a federated learning system can be divided into computation cost and communication cost, and we define it as follows:
\label{def:cost}
\begin{align*}
    \epsilon_c = \frac{1}{K}\sum_{k=1}^K (Q_{\text{comp}}+Q_{\text{comm}}),
\end{align*}
where $Q_{\text{comp}}$ measures computation cost, while $Q_{\text{comm}}$ measures communication cost.

\section{Constrained Multi-Objective Federated Learning Algorithms}\label{sec:cmofl-algos}

In this section, we provide an overview of a gradient-free multi-objective federated learning (MOFL) workflow and propose two improved MOFL algorithms for finding the Pareto optimal solutions that minimize privacy leakage, utility loss, and training cost. 

\begin{figure*}[ht!]
\centering
\includegraphics[width=0.96\linewidth]{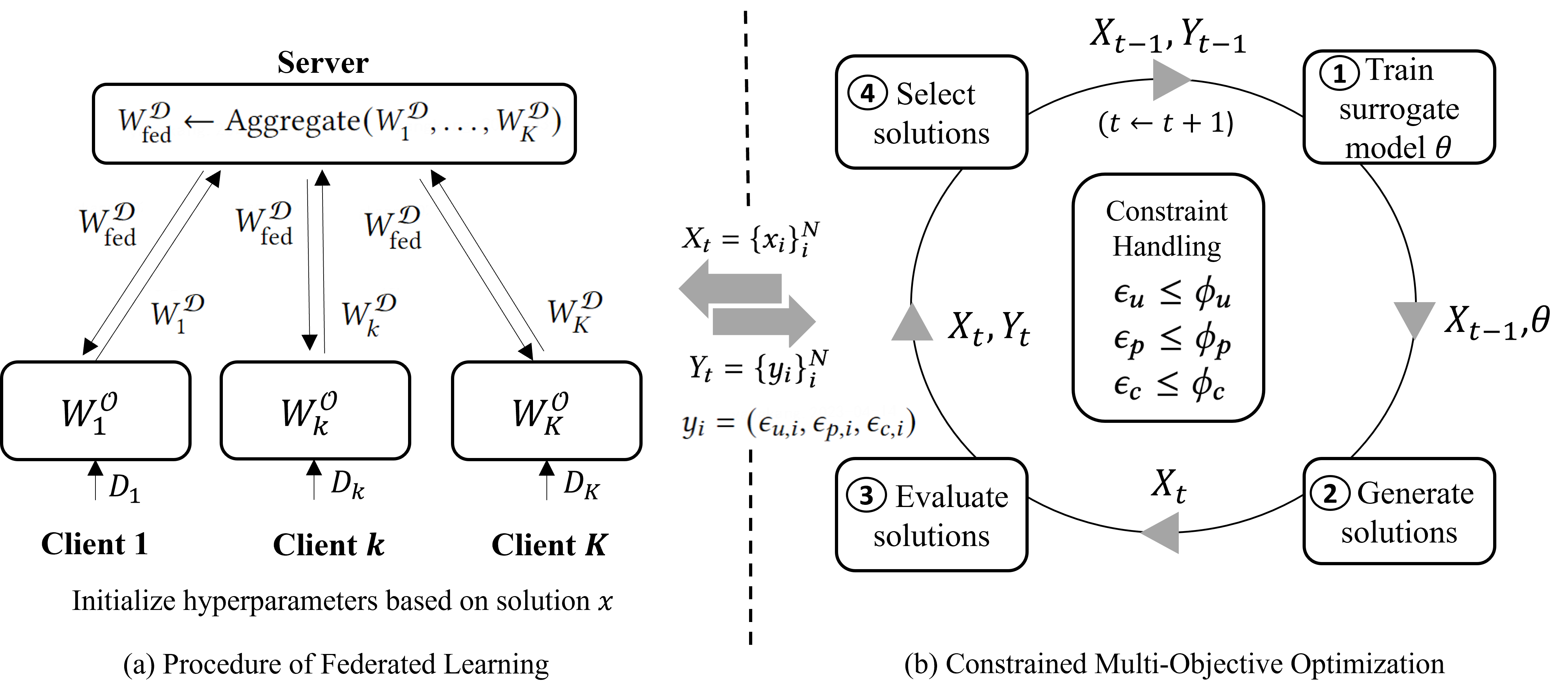}
\small
\vspace{-1em}
\caption{The Constrained Multi-Objective Federated Learning Workflow. The left panel illustrates a typical federated learning training procedure. The right panel gives the general constrained multi-objective optimization procedure involving five sub-procedures: train surrogate models, generate solutions, evaluate solutions, select solutions, and constraint handling. The select solutions sub-procedure and the evaluate solutions sub-procedure may call the federated learning procedure to obtain utility loss $\epsilon_u$, privacy leakage $\epsilon_p$, and training cost $\epsilon_c$ by evaluating a given solution $x \in \mathcal{X}$.} 
\label{fig:mofl}
\end{figure*}

Figure \ref{fig:mofl} illustrates the workflow of the general CMOFL, which consists of (i) a federated learning procedure that evaluates solutions and (ii) a multi-objective optimization (MOO) procedure that finds Pareto optimal solutions under constraints based on objective values. 

\begin{itemize}
    \item The federated learning procedure, serving as an evaluation function, evaluates given solutions to measure their corresponding objective values for privacy leakage, utility loss, and training cost. Algorithm  \ref{alg:fl} describes the detailed FL procedure in MOFL. It follows the conventional secure FL~\cite{zhang2022no} training procedure but with additional steps of measuring privacy leakage, utility loss, and training cost. More specifically, each client measures privacy leakage and training cost after local training and uploads them to the server (line 10-11 in Algo. \ref{alg:fl}), which in turn aggregates uploaded information (line 14 in Algo. \ref{alg:fl}). Upon the completion of training, the server measures the final privacy leakage, training cost, and utility loss, and then sends them to the MOO procedure (line 16-17 in Algo. \ref{alg:fl}).
    \item The MOO procedure typically involves five sub-procedures: train surrogate models, generate candidate solutions, evaluate solutions, select solutions, and constraint handling. Different MOFL algorithms implement the five sub-procedures differently and may call them in different orders. Surrogate models are involved in surrogate-based MOO algorithms such as multi-objective Bayesian optimization. They enhance optimization efficiency by assigning solutions with surrogate objective values instead of calling expensive objective functions. Evaluate solutions or select solutions subprocedure may need to call the federated learning procedure to obtain the real values of FL objectives. Constraint handling subprocedure eliminates solutions that violate constraints.
\end{itemize}

\begin{algorithm}[!ht]\vspace{-3pt}
    \caption{Federated Learning Optimization (FLO)}
	\begin{algorithmic}[1]
    \vspace{2pt}
    \Statex \textbf{Input:} Dataset $D_k$ owned by client $k \in [K]$, solutions $X$, protection mechanism $\mathcal{M}$.
    \Statex \textbf{Output:} objective values $Y$ for $X$ \vspace{4pt}
    \For{each solution $x$ $\in$ $X$}
    \State set global model structure and hyperparameters according to $x$;
    \State initialize global model parameters $W_{\text{fed}}$;
    \For{each communication round $i \in I$}
    \State \gray{$\triangleright$ \textit{Clients perform:}}
    \For{client $k\in [K]$}
    \State decode $W_{\text{fed},i} \leftarrow W_{\text{fed},i}^\calD$ (optionally) and set $W_{k,i} \leftarrow W_{\text{fed},i}$;
    \State Local Update $W_{k,i} \leftarrow W_{k,i} - \eta \nabla \calL_{k,i}$ based on $D_k$;
    \State Apply protection mechanism to obtain $W_{k,i}^\calD \leftarrow \mathcal{M}(W_{k,i})$;
   \State Measure objectives of privacy leakage $\epsilon_{p,k,i}$ and learning cost $\epsilon_{e,k,i}$;
    \State Upload the $W_{k,i}^\calD$, $\epsilon_{p,k,i}$ and $\epsilon_{e,k,i}$ to the server;
     \EndFor
     \State \gray{$\triangleright$ \textit{Server perform:}}
     \State $W_{\text{fed},i+1}^\calD$ $\leftarrow \frac{1}{K}\sum_{k=1}^{k,i} W_{k,i}^\calD$;
     \State $\epsilon_{p,i}$ $\leftarrow$ Aggregate($\epsilon_{p,k,i}, k\in[K]$);$\epsilon_{c,i}$ $\leftarrow$ Aggregate($\epsilon_{c,k,i}, k\in[K]$)
     \State Distribute the $W_{\text{fed},i+1}^\calD$ to all clients.
     \EndFor
    \State $\epsilon_{p}$ $\leftarrow$ Aggregate($\epsilon_{p,i}, i\in[I]$);\,$\epsilon_{c}$ $\leftarrow$ Aggregate($\epsilon_{c,i}, i\in[I]$)
    \State Evaluate test accuracy of $W_{\text{fed}}^\calD$ on the test dataset and calculate utility loss $\epsilon_u$.
    \State $Y \leftarrow Y + (\epsilon_p, \epsilon_c.\epsilon_u)$
\EndFor
    \State \Return $Y$;
	\end{algorithmic}\label{alg:fl}
\end{algorithm}

Built upon the general MOFL workflow, we propose two constrained multiple-objective federated learning (CMOFL) algorithms to find better Pareto optimal solutions that satisfy constraints. To this end, both algorithms leverage a regret function to penalize solutions that violate constraints on privacy or training cost during optimization. The first CMOFL algorithm is based on NSGA-II~\cite{NSGA-II}, and thus we name it CMOFL-NSGA-II. The second one is based on PSL (Pareto Set Learning)~\cite{linParetoSetLearning2022}, a multi-objective Bayesian optimization algorithm, and we name it CMOFL-PSL. 

\subsection{CMOFL-NSGA-II}

Non-dominated Sorting Genetic Algorithm II (NSGA-II) is a well-known multi-objective evolutionary algorithm proposed by~\cite{NSGA-II}. NSGA-II follows the general procedure of the genetic algorithm and is characterized by a fast non-dominated sorting approach and diversity preservation heuristics based on crowding distance. The fast non-dominated sorting approach sorts a population into different non-dominated levels, and the crowding distance measures the closeness among individuals (i.e., solutions) at the same dominance level. Individuals with a higher non-dominated level and larger crowding distance are more likely to be selected to enter the next generation. 

\begin{algorithm}[!h]
	\caption{CMOFL-NSGA-II}
	\begin{algorithmic}[1]
	\Statex \textbf{Input:} generations $T$, datasets $\{D_k\}_{k=1}^K$ owned by $K$ clients; constraints $\phi_u, \phi_p, \phi_c$; penalty coefficients $\alpha_u, \alpha_p, \alpha_c$.
    \Statex \textbf{Output:} Pareto optimal solutions and Pareto front $\{X_{T}, Y_{T}\}$
    \State Initialize solutions $\{X_{0}\}$.
    \For{each generation $t$ $=1,2,\cdots,T$} 
        \State Crossover and mutate parent solutions $X_{t-1}$ to produce offspring solutions $P$; 
        \State $R$ $\leftarrow$ Merge $X_{t-1}$ and $P$;
        \State $Y$ $\leftarrow$ FLO$(R, \{D_k\}_{k=1}^K)$ 
        \Comment{\gray{use FL to obtain objective values}}
        

        \For{each tuple ($\epsilon_u, \epsilon_p, \epsilon_c$) in $Y$} \Comment{\gray{penalize objectives that violate constraints}}
        \begin{equation}
          \epsilon_{i} = \epsilon_{i} \text{ + } \alpha_i \max\{0, \epsilon_i-\phi_i\}, i \in \{u, p, c\}
        \end{equation}
        \EndFor
        
        \State $R^S \leftarrow$ Non-dominated sorting and crowding distance sorting $R$ based on $Y$;
        \State $X_{t}$ $\leftarrow$ Select $N$ high-ranking solutions from $R^S$;
    \EndFor
         \State \Return $\{X_{T}, Y_{T}\}$
	\end{algorithmic}\label{alg:nsga_fl}
\end{algorithm}

The FL version of NSGA-II was proposed in the work~\cite{Hangyu2020nsgafl}. Our algorithm CMOFL-NSGA-II improves it by applying constraints to objectives for finding better and more practical Pareto optimal solutions. 

Algo. \ref{alg:nsga_fl} describes the training process of CMOFL-NSGA-II. The first generation starts from initial solutions $X_0$. In each generation, offspring solutions $P$ are generated by crossover and mutation on parent solutions. Then, parent solutions and their offspring solutions are merged to form current solutions $R$, the objective values $Y$ of which are obtained by calling federated learning optimization procedure (Line 5 in Algo \ref{alg:nsga_fl}). We loop over objective values in $Y$ and add penalties to those violating prespecified constraints (Line 6 in Algo \ref{alg:nsga_fl}). Next, non-dominated sorting and crowding distance sorting are performed on $R$ based on $Y$, resulting in $R^S$. Last, $N$ solutions with the highest ranking in $R^S$ are selected to enter the next generation. Upon completing $T$ generations of evolution, the algorithm returns Pareto optimal solutions and their corresponding Pareto Front.

We provide the convergence analysis for Algo. \ref{alg:nsga_fl} as follows: 


\begin{lemma} \label{lem:lemm4-maintxt} The work \cite{zheng2022first}
considers the following LOTZ and ONEMINMAX benchmarks in two multi-objective problems. Let $d$ be the dimension of the solution space.
\begin{itemize}
    \item \textbf{LOTZ:} If the population size $N$ is at least $5(d+1)$,
then the expected runtime is $O(d^2
)$ iterations and $O(Nd^2)$
fitness evaluations.
\item \textbf{ONEMINMAX:} if the population size $N$ is at least $4(d + 1)$, then the
expected runtime is $O(d \text{log} d)$ iterations and $O(Nd \text{log} d)$
fitness evaluations.
\end{itemize}
\end{lemma}
Lemma \ref{lem:lemm4-maintxt} demonstrates the NSGA-II algorithm could obtain almost Pareto optimal solutions (i.e., within a small $\epsilon$ error) for LOTZ and ONEMINMAX benchmarks with sufficiently large population size $N$. Further, we provide the convergence analysis of Algo. \ref{alg:nsga_fl} when the objective values obtained by Algo. \ref{alg:nsga_fl} approaches the finite Pareto optimal objective values within $\epsilon$ error from the perspective of hypervolume.
\begin{theorem} \label{thm:thm1}
Let  $Y^*$ be the finite Pareto optimal objective values w.r.t $m$ objectives. If for any $y^* \in Y^*, \exists y^T \in Y_T$ obtained by Algo. \ref{alg:nsga_fl}, s.t.  $\|y^T-y^*\| \leq \epsilon$, then we have:
\begin{equation}
    \text{HV}_z(Y^*)-\text{HV}_z(Y_T) \leq Cm\epsilon,
\end{equation}
where $\text{HV}_z(\cdot)$ represents the hypervolume with reference point $z$, $m$ is the number of objectives and $C$ is a constant. See proof in Appendix B.1.
\end{theorem}

\subsection{CMOFL-PSL}


Pareto Set Learning (or PSL)~\cite{linParetoSetLearning2022} is a multi-objective Bayesian optimization (MOBO) algorithm. It learns a Pareto set model to map any valid preference to corresponding solutions and builds independent Gaussian process models to approximate each expensive objective function. Based on the Pareto set model and Gaussian process model, surrogate objectives are scalarized using the weighted Tchebycheff approach. Solving the Tchebycheff scalarized subproblem with specific trade-off preferences is equivalent to finding Pareto optimal solutions.



\begin{algorithm}[!ht]\vspace{-3pt}
	\caption{CMOFL-PSL}
	\begin{algorithmic}[1]
    \vspace{2pt}
	\Statex \textbf{Input:} generations $T$, datasets $\{D_k\}_{k=1}^K$ owned by $K$ clients; constraints $\phi_u, \phi_p, \phi_c$; penalty coefficients $\alpha_u, \alpha_p, \alpha_c$.
    \Statex \textbf{Output:} Pareto front $\{X_{T}, Y_{T}\}$ \vspace{4pt}
    \State Initialize solutions $\{X_{0},Y_{0}\}$.
   \For{each generation $t$ $=1,2,\cdots,T$} \vspace{2pt}
        \State Train Pareto Set model $h_\theta$ and Gaussian process model $g$ using $\{X_{t-1}, Y_{t-1}\}$; 
        \State $R$ $\leftarrow$ Generate $m$ candidate solutions using $h_\theta$ 
        \State Compute surrogate objective values for candidate solutions $R$: $\hat{Y}$ $\leftarrow$ $g(R)$;


        \For{each tuple ($\epsilon_u, \epsilon_p, \epsilon_c$) in $\hat{Y}$} \Comment{\gray{penalize objectives that violate constraints}}
            \begin{equation}
              \epsilon_i = \epsilon_i + \alpha_i \max\{0, \epsilon_i-\phi_i\}, i \in \{u, p, c\}
            \end{equation}
        \EndFor
        
        \State $X$ $\leftarrow$ Select $N$ solutions from $R$ with the highest hypervolume improvement based on $\hat{Y}$;        
        \State $Y \leftarrow$ FLO$(X, \{D_k\}_{k=1}^K)$ 
        \Comment{\gray{use FL to obtain real objective values}}
        
        \State $\{X_{t}, Y_{t}\}$ $\leftarrow$ $\{X_{t-1}, Y_{t-1}\}$ + $\{X, Y\}$; 
    \EndFor
         \State \Return $\{X_{T}, Y_{T}\}$
	\end{algorithmic}\label{alg:psl_fl}
\end{algorithm}

We implement our FL version of PSL based on work~\cite{linParetoSetLearning2022} and name it CMOFL-PSL. Algo. \ref{alg:psl_fl} describes the training process of CMOFL-PSL: In each generation, a Pareto Set model $h_\theta$ and the Gaussian process model $g$ are first trained with the accumulated Pareto optimal solutions and front $\{X_{t-1}, Y_{t-1}\}$ (Line 3 in Algo. \ref{alg:psl_fl}. We refer readers to work~\cite{linParetoSetLearning2022} for detail). Then, $m$ ($m \gg N$) candidate solutions $R$ are generated by $h_\theta$, and their surrogate objective values $\hat{Y}$ are derived by the Gaussian process $g$. We loop over objective values in $\hat{Y}$ and add penalties to the ones that violate prespecified constraints (Line 6 in Algo. \ref{alg:psl_fl}). After that, a greedy batch selection process is applied to $R$ for selecting top $N$ solutions $X$ that have the biggest hypervolume improvement based on $\hat{Y}$. Next, the selected solutions $X$ are evaluated by calling the federated learning optimization procedure to obtain their real objective values $Y$ (Line 8 in Algo. \ref{alg:psl_fl}). Last, $X$ and $Y$ are appended to accumulated Pareto optimal solutions and front $\{X_{t-1}, Y_{t-1}\}$, resulting in $\{X_{t}, Y_{t}\}$. Upon completing $T$ generations, the algorithm returns Pareto optimal solutions and their corresponding Pareto Front.

We analyze the convergence of Algo. \ref{alg:psl_fl} as follows: 

\begin{theorem} \label{thm:conver-PSL} \cite{zhang2020random} If the type of scalarizations method is hypervolume scalarization, the hypervolume regret after $T$ observations (i.e., generations) is upper bounded by:
\begin{equation}
  \sum_{t=1}^T\big(\text{HV}_z(Y^*)- \text{HV}_z(Y_t)\big)\leq O(m^2d^{1/2}[\gamma_Tln(T)T]^{1/2}),
\end{equation}  
where $Y_t$ is obtained by Algo. \ref{alg:psl_fl} and $\gamma_T$ is a kernel-dependent quantity known as the maximum information gain. For example,
$\gamma_T = O(\text{poly }ln(T))$ for the squared exponential kernel. Furthermore,  $\text{HV}_z(Y^*)- \text{HV}_z(Y_T) \leq O(m^2d^{1/2}[\gamma_Tln(T)/T]^{1/2})$.
\end{theorem}  


Theorem \ref{thm:conver-PSL} demonstrates the convergence of Algo. \ref{alg:psl_fl} according to the hypervolume. Specifically, the convergence bound increases with the increment of the number of objectives $m$ or the dimension $d$ of a solution and decreases with the increment of $T$. The multi-objective Bayesian optimization in Algo. \ref{alg:psl_fl} (CMOFL-PSL) \cite{linParetoSetLearning2022} leverages the Tchebyshev scalarizations, whose relation to Hypervolume scalarizations is presented in Appendix B.2.

\section{Experiment} \label{sec:exp}


In this section, we elaborate on the empirical experiments to verify our proposed algorithms.



\subsection{Experimental Setups}\label{sec:exp_setup}
This section details the experimental setups, including the datasets and models we adopt to run experiments, federated learning setup, multi-objective optimization algorithm setup, and experimental settings in which we conduct our experiments.

\subsubsection{Datasets and Models}
We conduct experiments on two datasets:
\textit{Fashion-MNIST}~\cite{xiao2017fmnist} and \textit{CIFAR10} ~\cite{krizhevsky2014cifar}. Fashion-MNIST has 60000 training data and 10000 test data, while CIFAR10 has 50000 training data and 10000 test data. Both datasets have 10 classes. We adopt MLP (Multilayer Perceptrons) for conducting experiments on Fashion-MNIST and a modified LeNet \cite{lecun1998gradient} on CIFAR10. 

The MLP consists of 2 hidden layers and a softmax layer. Each hidden layer has 256 neurons, and the softmax layer has 10 neurons. The modified LeNet consists of 2 convolutional layers, 2 hidden layers, and a softmax layer. Each convolutional layer has 32 channels and a kernel size of $5\times5$, each hidden layer has 128 neurons, and the last softmax layer has 10 neurons. All activation functions are ReLU. Table \ref{table:data_model} summarizes the structures of the two models.
\begin{table}[!htbp]
	\centering
	\caption{Datasets and models for experiments. $ks$ is kernel size, $fm$ is the number of feature maps.}
 \vspace{-0.5em}
	\begin{tabular}{c||c|c|c}
	        \hline
		Dataset & Model  & Convolutional layers & Fully-connected layers  \\
         \hline
         \hline
            \\[-1em]
		Fashion-MNIST & MLP & \textemdash &  $256\rightarrow256\rightarrow10$   \\
		\hline   
          \\[-1em]
            CIFAR10 &  LeNet & ($ks: 5\times5, fm: 32$) $\rightarrow$ ($ks: 5\times5, fm: 32$) &  $128\rightarrow128\rightarrow10$ \\
		\hline      
	\end{tabular}
\label{table:data_model}
\end{table}
\vspace{-10pt}

We also consider the neural network structure as a dimension to optimize for trading off privacy leakage, utility loss, and training cost. The detail is covered in Sec. \ref{sec:exp_settings}.


\subsubsection{Federated Learning Setup}
We conduct federated learning with 5 clients and focus on IID (independent and identically distributed) data setting. We assign each client 12000 training data for Fashion-MNIST and 10000 training data for CIFAR10 to train their local models and use 10000 test data to validate the aggregated global model. The federated optimization is performed with 10 global communication rounds among clients and 5 local epochs with a batch size of 64 on each client. All clients use the SGD optimizer to train their local models. We simulate federated learning in standalone mode.

\subsubsection{Baseline}
We consider the FL version of NSGA-II~\cite{NSGA-II} and PSL~\cite{linParetoSetLearning2022} as our baselines. We name them MOFL-NSGA-II and MOFL-PSL, respectively. NSGA-II is a well-known multi-objective evolutionary algorithm, while PSL is a novel multi-objective Bayesian optimization algorithm. We implement their FL version based on works~\cite{Hangyu2020nsgafl,linParetoSetLearning2022}. MOFL-NSGA-II and MOFL-PSL do not handle constraints. For a fair comparison, we report the Pareto optimal solutions satisfying constraints generated from MOFL-NSGA-II and MOFL-PSL.

\subsubsection{Multi-Objective Optimization algorithm Setup}

For all CMOFL methods compared in this work, we set the total number of generations (i.e., iterations) to 20 and the population size to 20 unless otherwise specified. 

\textbf{NSGA-II Setup.} We follow literature~\cite{Hangyu2020nsgafl} to set NSGA-II parameters. 
For binary chromosome, we apply a single-point crossover with a probability of 0.9 and a bit-flip mutation with a probability of 0.1. For real-valued chromosome, we apply a simulated binary crossover (SBX)~\cite{Deb1995RealcodedGA} with a probability of 0.9 and $n_c$ = 2, and a polynomial mutation with a probability of 0.1 and $n_m$ = 20, where $n_c$ and $n_m$ denote spread factor distribution indices for crossover and mutation, respectively.

\textbf{PSL Setup.}  We follow literature~\cite{linParetoSetLearning2022} to set PSL parameters. At each iteration, we train the Pareto set model $h_{\theta}$ with 1000 update steps using Adam optimizer with a learning rate of 1e-5 and no weight decay. At each iteration, we generate 1000 candidate solutions using $h_\theta$ and select the population size of solutions from the 1000 candidates.

\subsubsection{Experimental settings} \label{sec:exp_settings}
We use three experimental settings to investigate the effectiveness of our proposed algorithms: (1) use \textit{Randomization} (RD) to protect data privacy and minimize utility loss and privacy leakage; (2) use \textit{BatchCrypt} (BC)  to protect data privacy and minimize utility loss and training cost; (3) use \textit{Sparsification} (SF) to protect data privacy and minimize utility loss, privacy leakage, and training cost.
Table \ref{tab:exp_setting} summarizes the three experimental settings and following elaborates on the setup of each setting. 
\begin{table}[!h]  \renewcommand\arraystretch{1.2}
	\caption{The three experimental settings that apply Randomization (RD), BatchCrypt (BC), and Sparsification (SF), respectively, to protect data privacy. FC: fully-connected layer.}
	\label{tab:exp_setting}
\small
\vspace{-5px}
\begin{tabular}{c c||c|c|c}
			\hline  
  ~ & ~ & \multicolumn{3}{c}{Experimental Settings} \\ \hline
  ~ & ~ & RD & BC & SF \\
  \cline{3-5}
   \multicolumn{2}{c||}{\multirow{3}{*}{Objectives to minimize}} & \multirow{3}{*}{\shortstack{utility loss\\ privacy leakage}} & \multirow{3}{*}{\shortstack{utility loss\\ training cost}} & \multirow{3}{*}{\shortstack{utility loss\\ privacy leakage \\ training cost}} \\ 
   & & & & \\
   & & & & \\
   \hline
   \multicolumn{2}{c||}{\multirow{2}{*}{Constraint}} & \multirow{2}{*}{\shortstack{privacy leakage $\leq 0.8$ \\ $\alpha_p=20$}} &  \multirow{2}{*}{\shortstack{training cost $\leq 500\text{s}$ \\ $\alpha_c=20$}} & \multirow{2}{*}{\shortstack{privacy leakage $\leq 0.8$ \\ $\alpha_p=20$ }}\\
     & & & & \\
   \hline
   \hline
   \multirow{7}{*}{\shortstack{Solution \\variables}} & \multirow{2}{*}{\shortstack{Protection \\parameter}} & \multirow{2}{*}{\shortstack{$\sigma_{\text{rd}}:$[0, 1] \\ $c_{\text{clip}}:$[1, 4]}} & \multirow{2}{*}{$bs:$\{100, 200, 400, 800\}} & \multirow{2}{*}{\shortstack{$\rho:$[0, 1]; \\ $\xi:$[0, 0.99]; }} \\ 
   & & & & \\
   \cline{2-5}
   & Learning rate & [0.01, 0.3] & [0.01, 0.3] & [0.01, 0.3] \\ 
   \cline{2-5}
   & \multirow{2}{*}{\shortstack{for MLP: \\ \# of FC neurons}} & \multirow{2}{*}{\textemdash} & \multirow{2}{*}{[1, 256]} & \multirow{2}{*}{[1, 256]}  \\ 
   & & & & \\
   \cline{2-5}
   & \multirow{3}{*}{\shortstack{for LetNet:\\ \# of channels \\\# of FC neurons}} & \multirow{3}{*}{\textemdash} & \multirow{3}{*}{\shortstack{$[1, 32]$\\$[1, 128]$}} & \multirow{3}{*}{\shortstack{$[1, 32]$\\$[1, 128]$}}  \\ 
   & & & & \\
   & & & & \\
   \hline
	\end{tabular}
\end{table}

All three settings measure utility loss $\epsilon_u$ using test error:
\begin{equation} 
     \epsilon_u = 1 - \text{Acc}(W_{\text{fed}}^{\calD}),
\end{equation}
where $W_{\text{fed}}^{\calD}$ is the protected global model and $\text{Acc}(\cdot)$ evaluates the accuracy of $W_{\text{fed}}^{\calD}$ using test data. This measurement uses 1.0 as the utility upper bound and is a variant of Eq. (\ref{def:utility}).

\textbf{Randomization setting.} 
Randomization (RD) protects data privacy by adding Gaussian noise to each client's model parameters to be shared with the server. In this setting, we leverage the measurement provided in Eq. (\ref{eq:dp_privacy_measure}) to measure the privacy leakage $\epsilon_p$: 
\begin{equation} \label{eq:dp_privacy_measure}
     \epsilon_p = 1 - \min\{1, C_1\frac{\sigma_{\text{rd}}^2}{c_{\text{clip}}^2}\sqrt{d_w}\},
\end{equation}
where $\sigma_{\text{rd}}$ denotes the standard deviation of the Gaussian noise added to model parameters, $c_{\text{clip}}$ denotes the clip norm, and $d_w$ denotes the dimension of model parameters, $C_1$ is a prespecified constant and is set to 1 for Fashion-MNIST and CIFAR10. This measurement is derived from Eq. (\ref{eq: def_of_pl}). We refer readers to Appendix A for details.

In this setting, a solution contains the following variables to optimize: learning rate, standard deviation $\sigma_{\text{rd}}$, and clip norm $c_{\text{clip}}$ of Gaussian noise. We assume clients require that the privacy leakage is constrained not to go beyond $0.8$. Note that the dimension of model parameters $d_w$ serves as a constant and will not be optimized.

\textbf{BatchCrypt setting.} BatchCrypt (BC)~\cite{zhang2020batchcrypt} is a batch encryption technique that aims to improve the efficiency of homomorphic encryption. BatchCrypt quantizes model parameters into low-bit integer representations, encodes quantized parameters into batches, and then encrypts each batch in one go. 
Essentially, the BatchCrypt batch size $bs$ controls the trade-off between training cost and utility loss. We measure the training cost $\epsilon_c$ of BatchCrypt as follows:
\begin{equation}
    \epsilon_c = \frac{1}{K}\sum_{k=1}^K \left( Q_{\text{time}_1
    }(W^O_k, bs) + Q_{\text{time}_2}(\{W^{\calD}_k\}_{k=1}^K) \right)
\end{equation}
where $Q_{\text{time}_1}$ measures the time spent on training and encrypting a client model while $Q_{\text{time}_2}$ measures the time spent on aggregating encrypted clients' models.


In this setting, a solution contains the following variables to optimize: learning rate, the number of neurons in each layer for the two hidden-layer MLP (and the number of channels in each layer for LetNet), and BatchCrypt batch size $bs$ chosen in the range of $[100, 200, 400, 800]$. The training cost is constrained not to go beyond $500$ seconds.

\textbf{Sparsification setting.} 
We follow the sparsification (SF) method proposed in~\cite{Hangyu2020nsgafl}. Intuitively, the SF can be considered as a privacy protection mechanism that decomposes each client's model into a public sub-model and a private sub-model, then shares the public sub-model with the server while retaining the private sub-model locally to protect privacy. The size of the private sub-model is controlled by two hyperparameters: (1) $\rho$, the probability of connection between every two neighboring hidden layers; (2) $\xi$, the fraction of the model parameters with the smallest update to be retained at local.

Let $\mu_k$ denote the mean of values of \textit{retained} (or private) model parameters of $k_{th}$ client. 
We use the measurements provided in Eq. (\ref{equ:sf_privacy_measure}) to measure the privacy leakage $\epsilon_p$.

\begin{equation}\label{equ:sf_privacy_measure}
\epsilon_p = \frac{1}{K}\sum_{k=1}^K\left(1-\sqrt{2}(1-\exp\{- \frac{\mu_k}{C_2} \})^{1/2}\right),
\end{equation}
where $C_2$ is a prespecified constant, it is set to 8 for Fashion-MNIST and 32 for CIFAR10. This measurement is derived from Eq. (\ref{eq: def_of_pl}). We refer readers to Appendix A for the proof.

We measure training cost using the following measurement \cite{zhu2019multi}:
\begin{equation}\label{equ:sf_cost}
\epsilon_c = \frac{1}{K} \sum_{k}v_k,
\end{equation}
where $v_k$ is the number of \textit{shared} (or public) model parameters of $k_{th}$ client. A larger $v_k$ indicates more model parameters of client $k$ are shared with the server, meaning a higher training cost. 

In this setting, a solution contains the following variables to optimize: learning rate, number of neurons in each layer for MLP (or number of channels in each layer for LetNet), connection probability $\rho$ chosen in the range of $[0, 1]$ and 
$\xi$ chosen in the range of $[0, 0.99]$. Similar to RD, we assume that clients require the privacy leakage to be constrained not to go beyond 0.8.




\subsection{Main Experimental Results}\label{main_results}
We compare our proposed algorithms, CMOFL-NSGA-II and CMOFL-PSL, with their corresponding baselines, MOFL-NSGA-II and MOFL-PSL, in terms of hypervolume trends with respect to the number of generations for the three experimental settings (RD, BC, and SF). For RD and SF settings, the privacy leakage is constrained to be less than 0.8, while for SF, the training cost is constrained to be less than 500 seconds.

\begin{figure*}[!h]
\vspace{-3pt}
	\centering
      	\begin{subfigure}{0.32\textwidth}
  		 	\includegraphics[width=1\textwidth]{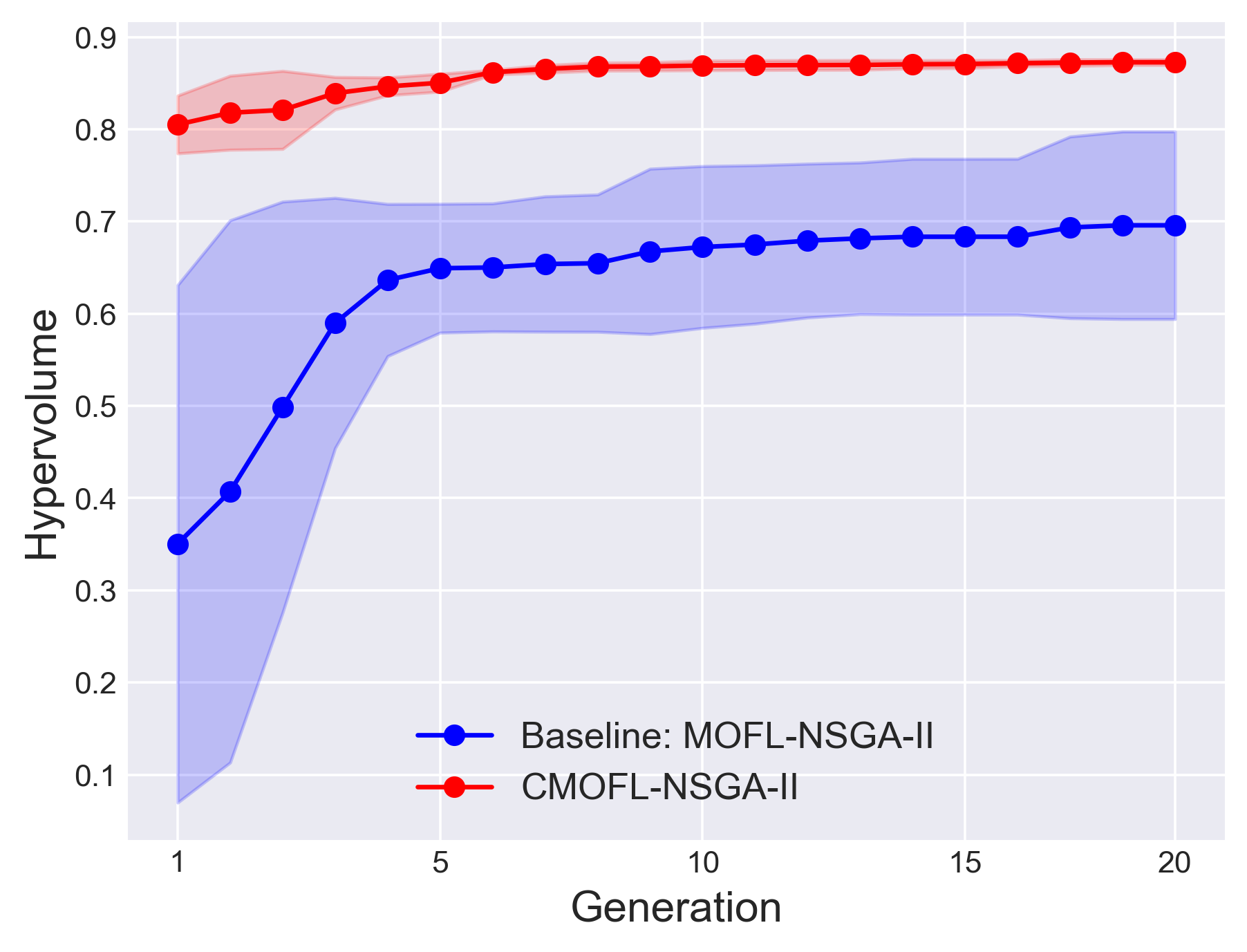}
      \subcaption{\scriptsize{BC: CMOFL-NSGA-II vs MOFL-NSGA-II}}
    		\end{subfigure}
    	\begin{subfigure}{0.32\textwidth}
  		 	\includegraphics[width=1\textwidth]{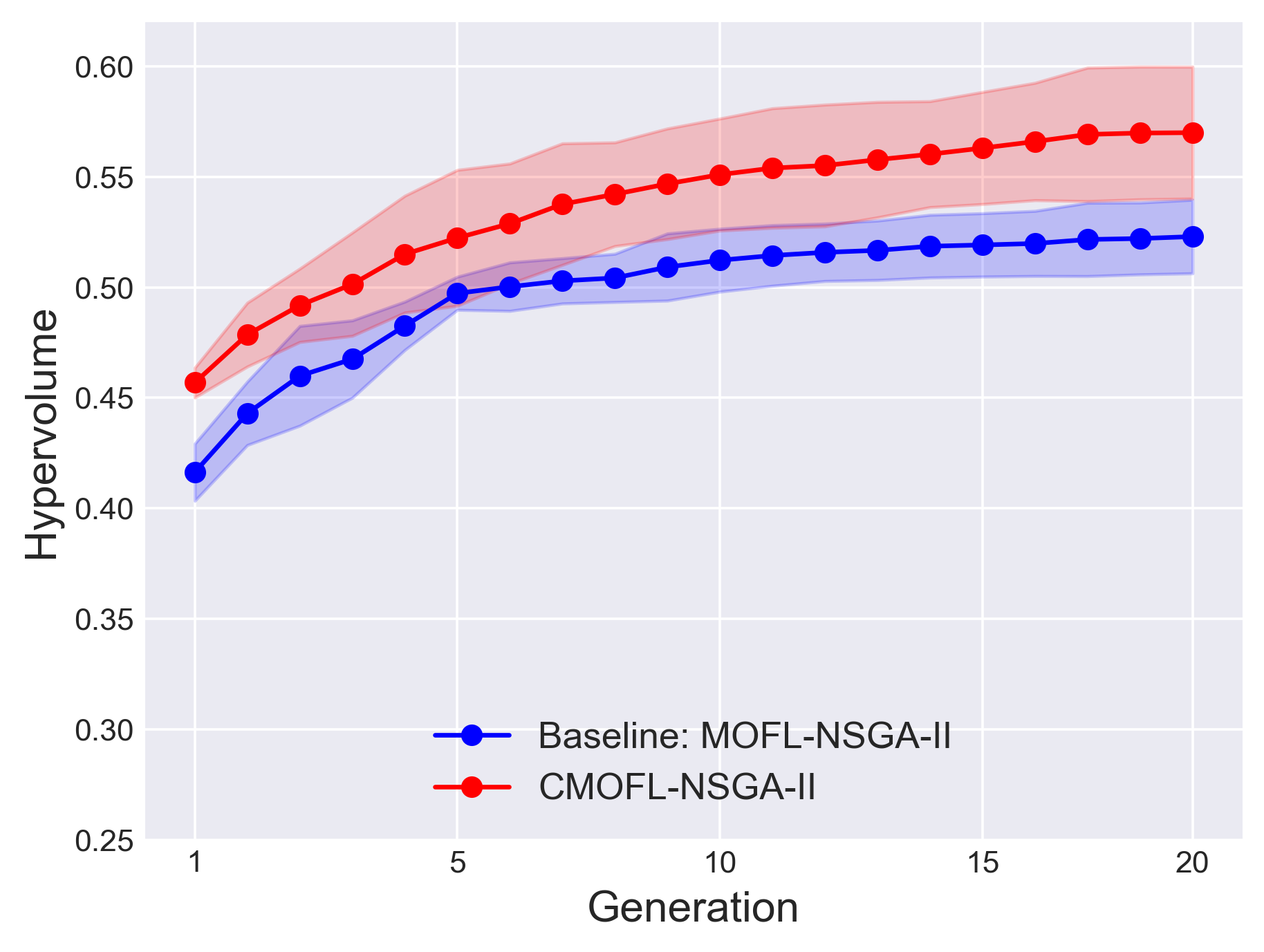}
            \subcaption{\scriptsize{RD: CMOFL-NSGA-II vs MOFL-NSGA-II}}
    		\end{subfigure}
   \begin{subfigure}{0.32\textwidth}
			\includegraphics[width=1\textwidth]{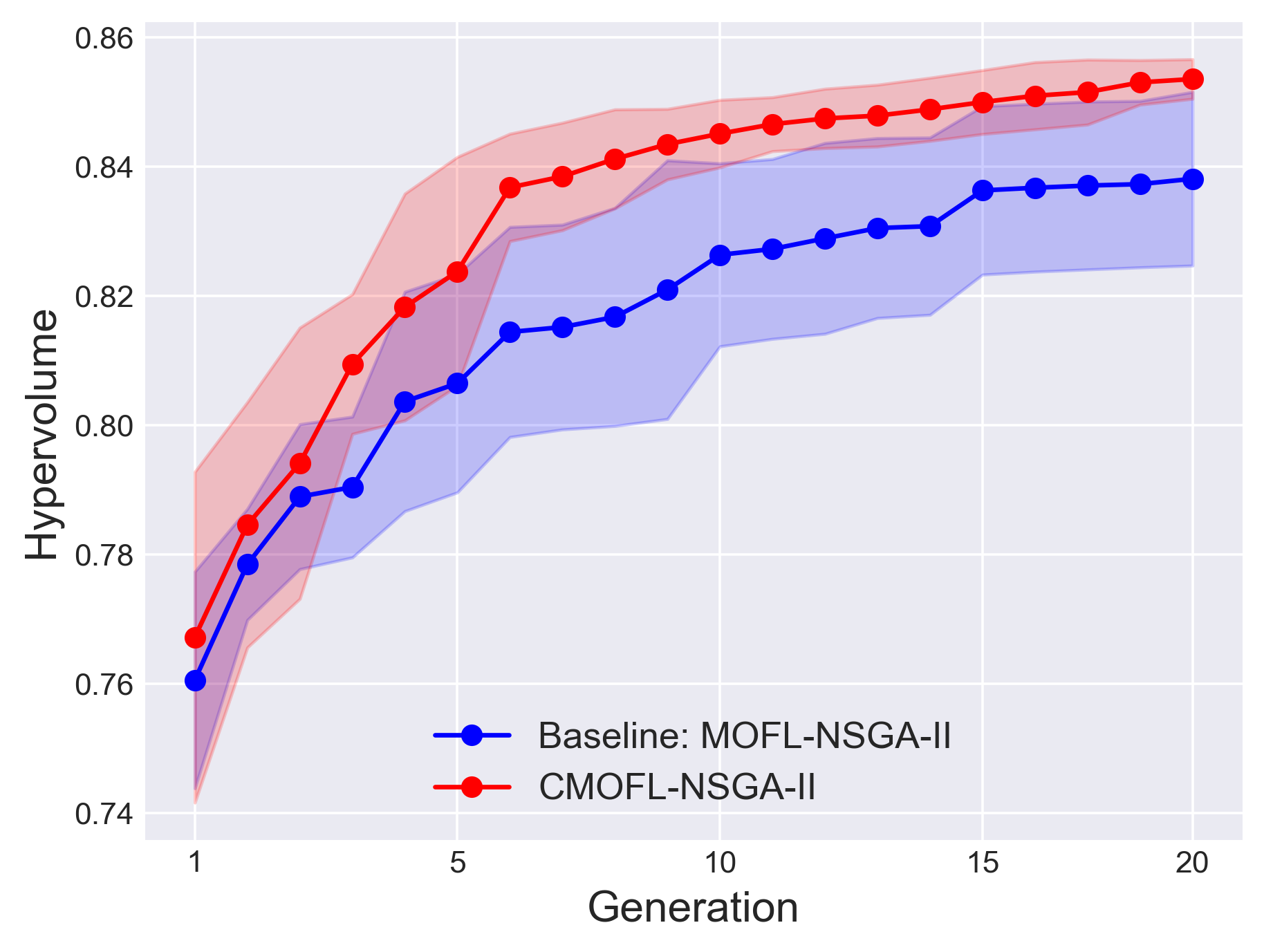}
         \subcaption{\scriptsize{SF: CMOFL-NSGA-II vs MOFL-NSGA-II}}
		\end{subfigure}

           \begin{subfigure}{0.32\textwidth}
  		 	\includegraphics[width=1\textwidth]{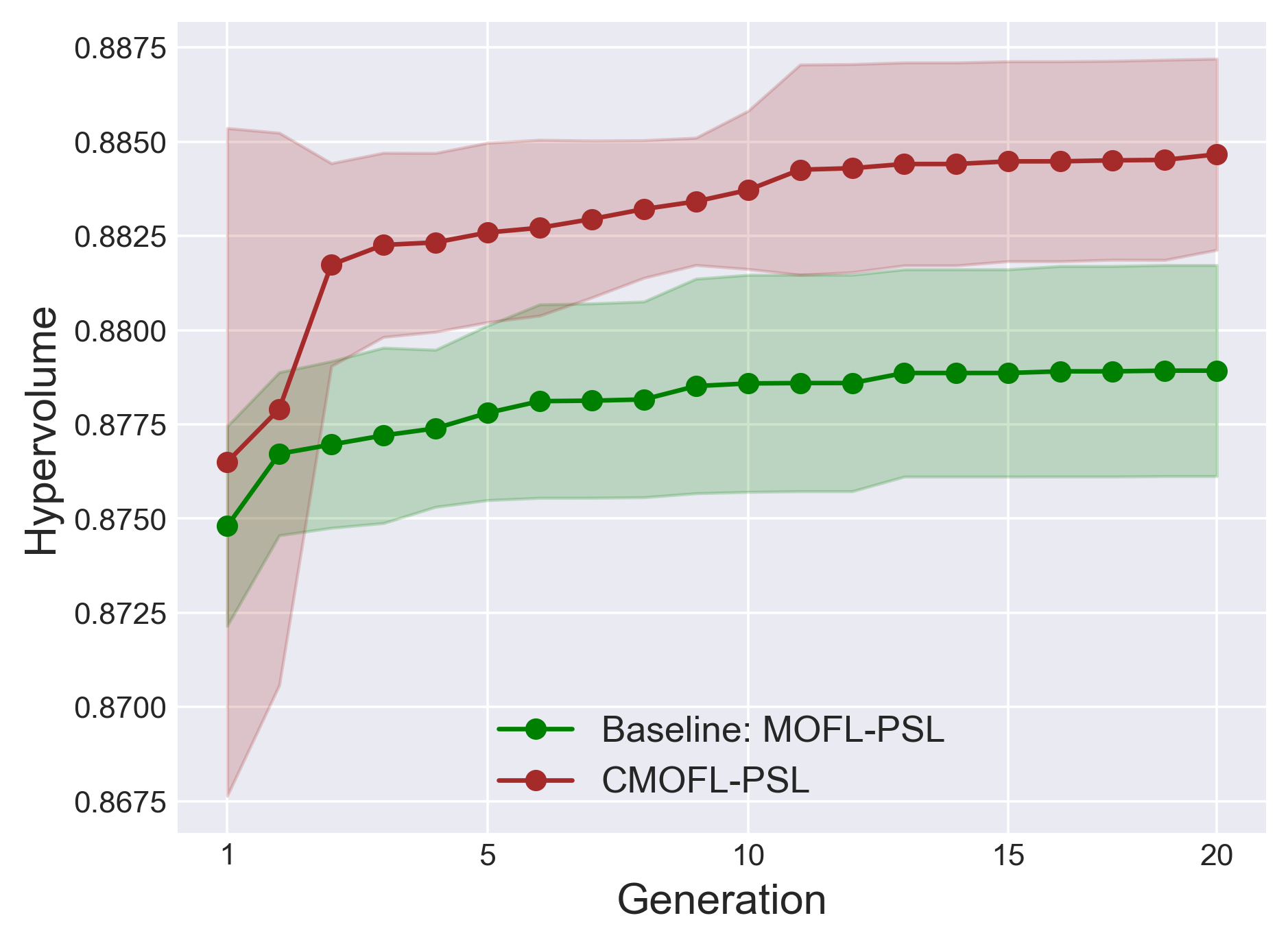}
      \subcaption{\scriptsize{BC: CMOFL-PSL vs MOFL-PSL}}
    		\end{subfigure}
    	\begin{subfigure}{0.32\textwidth}
  		 	\includegraphics[width=1\textwidth]{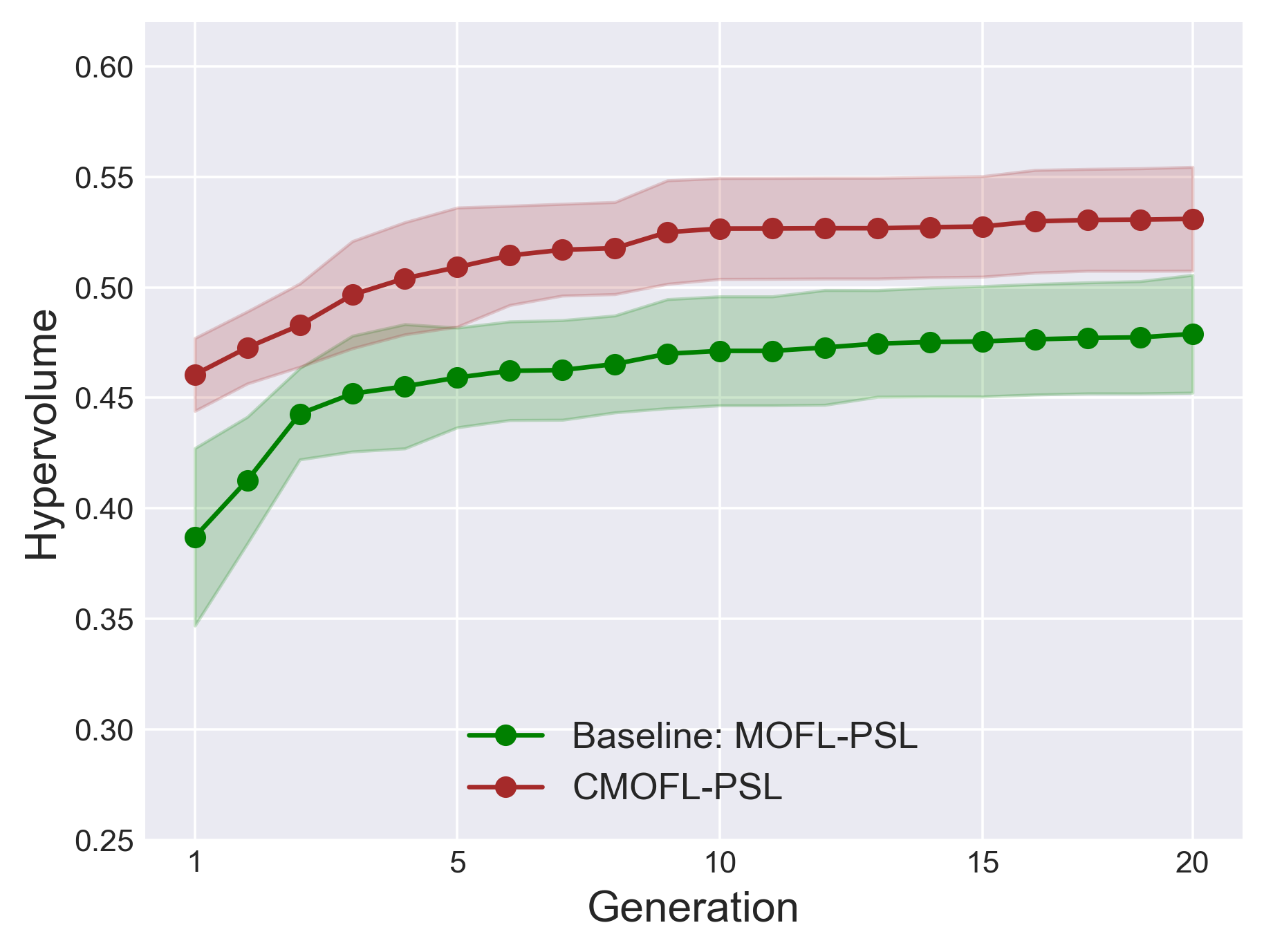}
            \subcaption{\scriptsize{RD: CMOFL-PSL vs MOFL-PSL}}
    		\end{subfigure}
   \begin{subfigure}{0.32\textwidth}
			\includegraphics[width=1\textwidth]{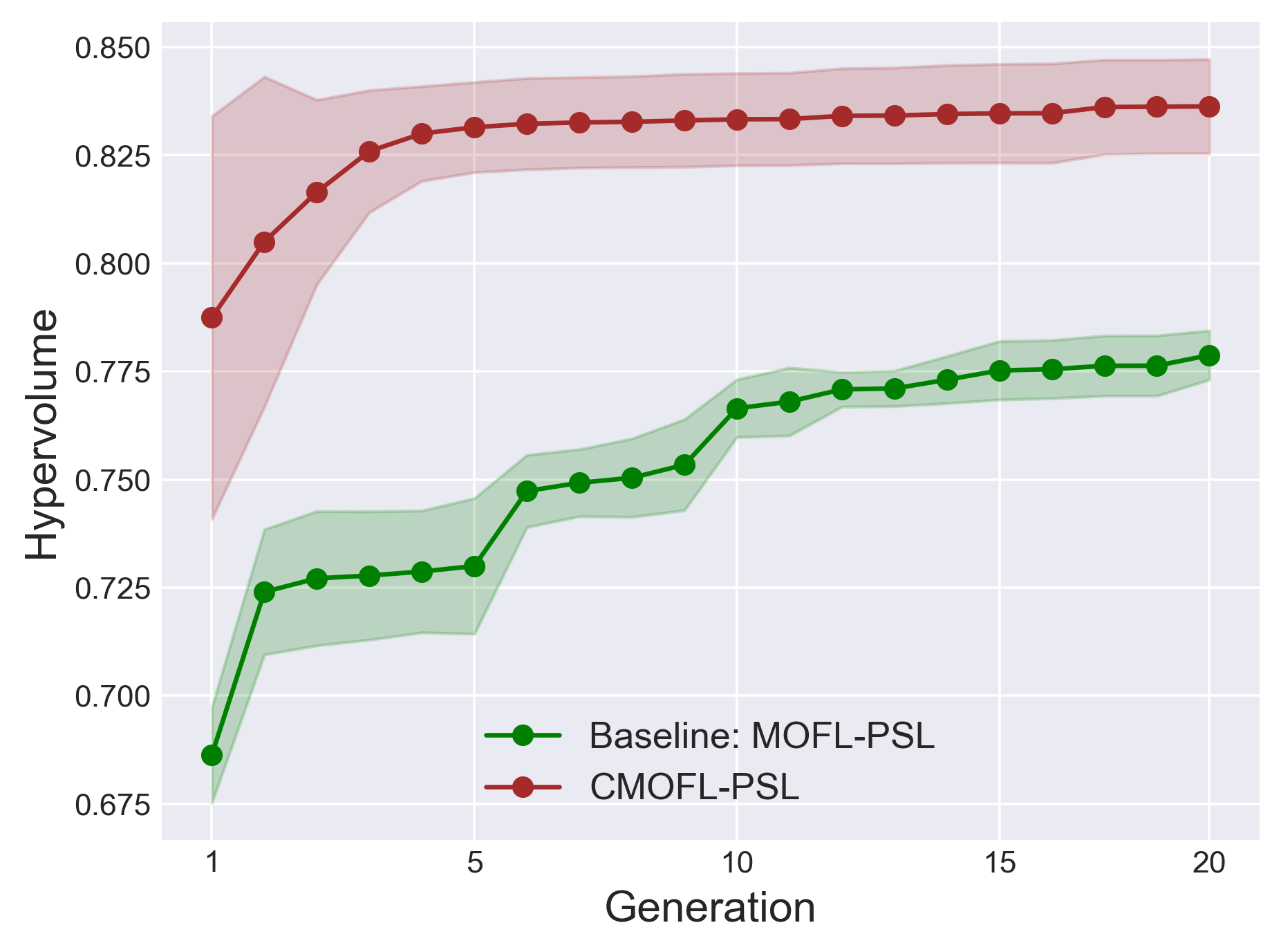}
         \subcaption{\scriptsize{SF: CMOFL-PSL vs MOFL-PSL}}
		\end{subfigure}
\vspace{-0.8em}
 \caption{Comparing hypervolume values of our proposed CMOFL algorithms with those of baseline MOFL algorithms on the Fashion-MNIST dataset for BC, RD, and SF, respectively. The first line compares CMOFL-NSGA-II and MOFL-NSGA-II for the three protection mechanisms. The second line compares CMOFL-PSL and MOFL-PSL for the three protection mechanisms.}
	\label{fig:hv_fmnist}
 \vspace{-4pt}
\end{figure*}

\begin{figure*}[!h]
\vspace{-3pt}
	\centering
      	\begin{subfigure}{0.32\textwidth}
  		 	\includegraphics[width=1\textwidth]{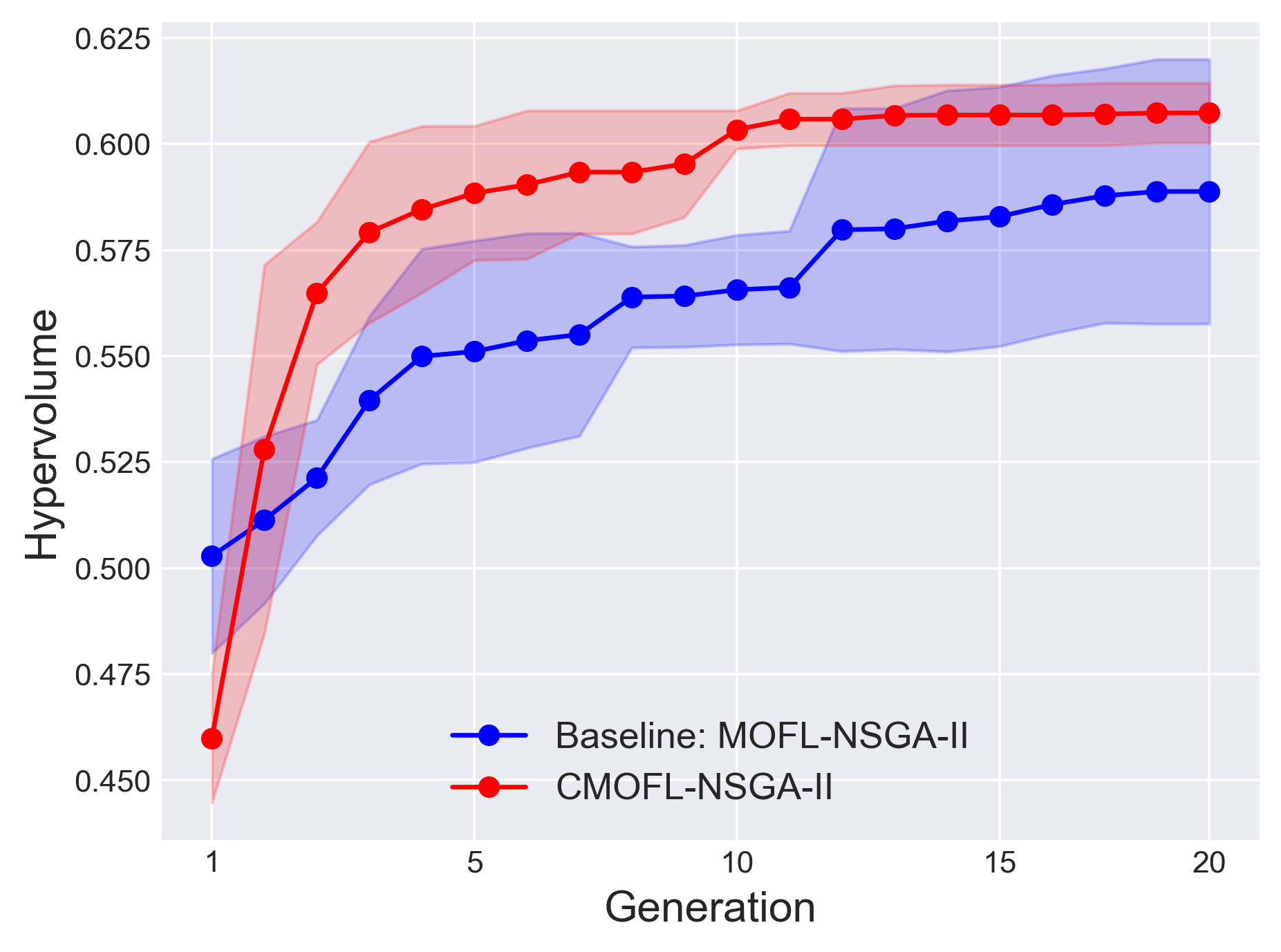}
      \subcaption{\scriptsize{BC: CMOFL-NSGA-II vs MOFL-NSGA-II}}
    		\end{subfigure}
    	\begin{subfigure}{0.32\textwidth}
  		 	\includegraphics[width=1\textwidth]{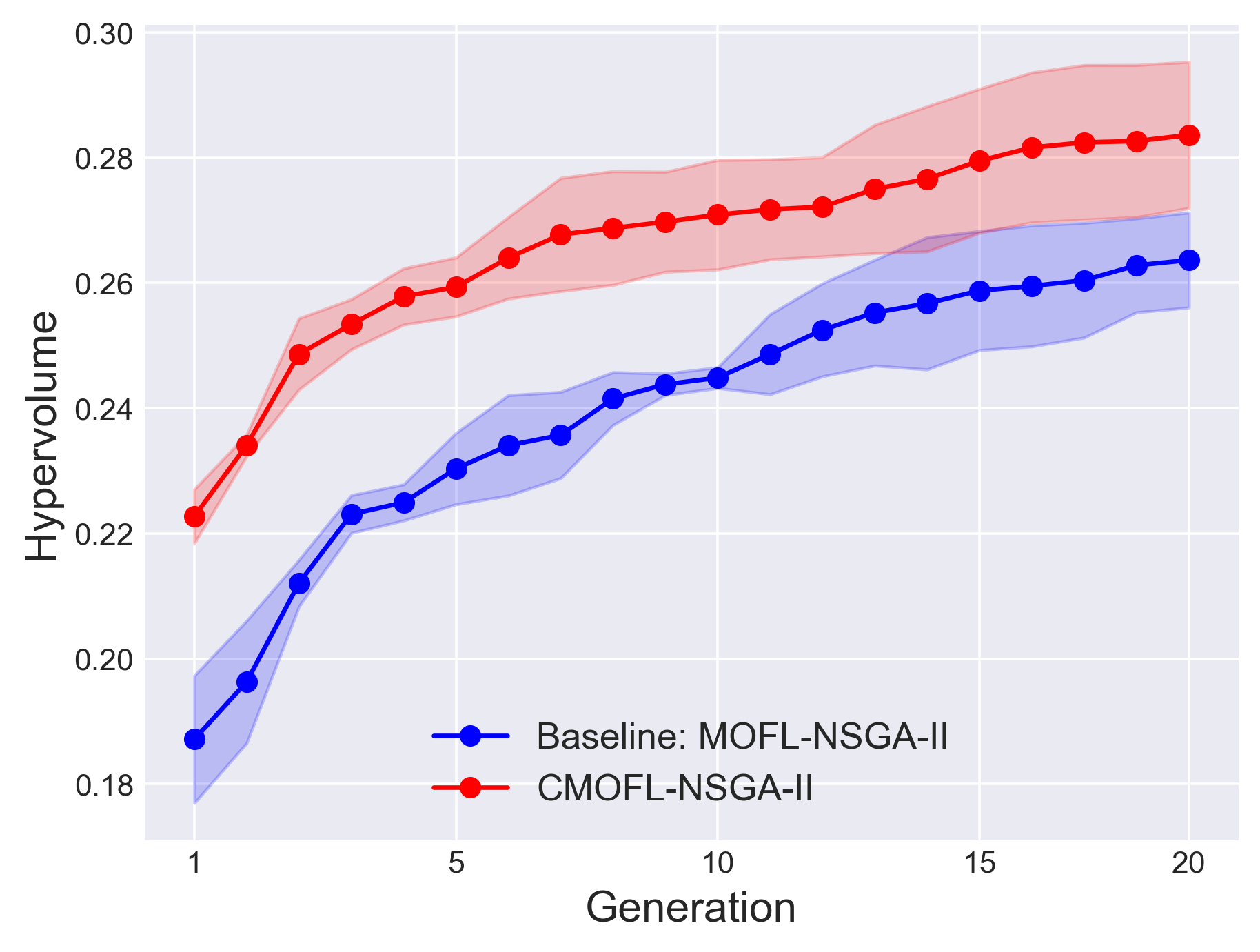}
            \subcaption{\scriptsize{RD: CMOFL-NSGA-II vs MOFL-NSGA-II}}
    		\end{subfigure}
   \begin{subfigure}{0.32\textwidth}
			\includegraphics[width=1\textwidth]{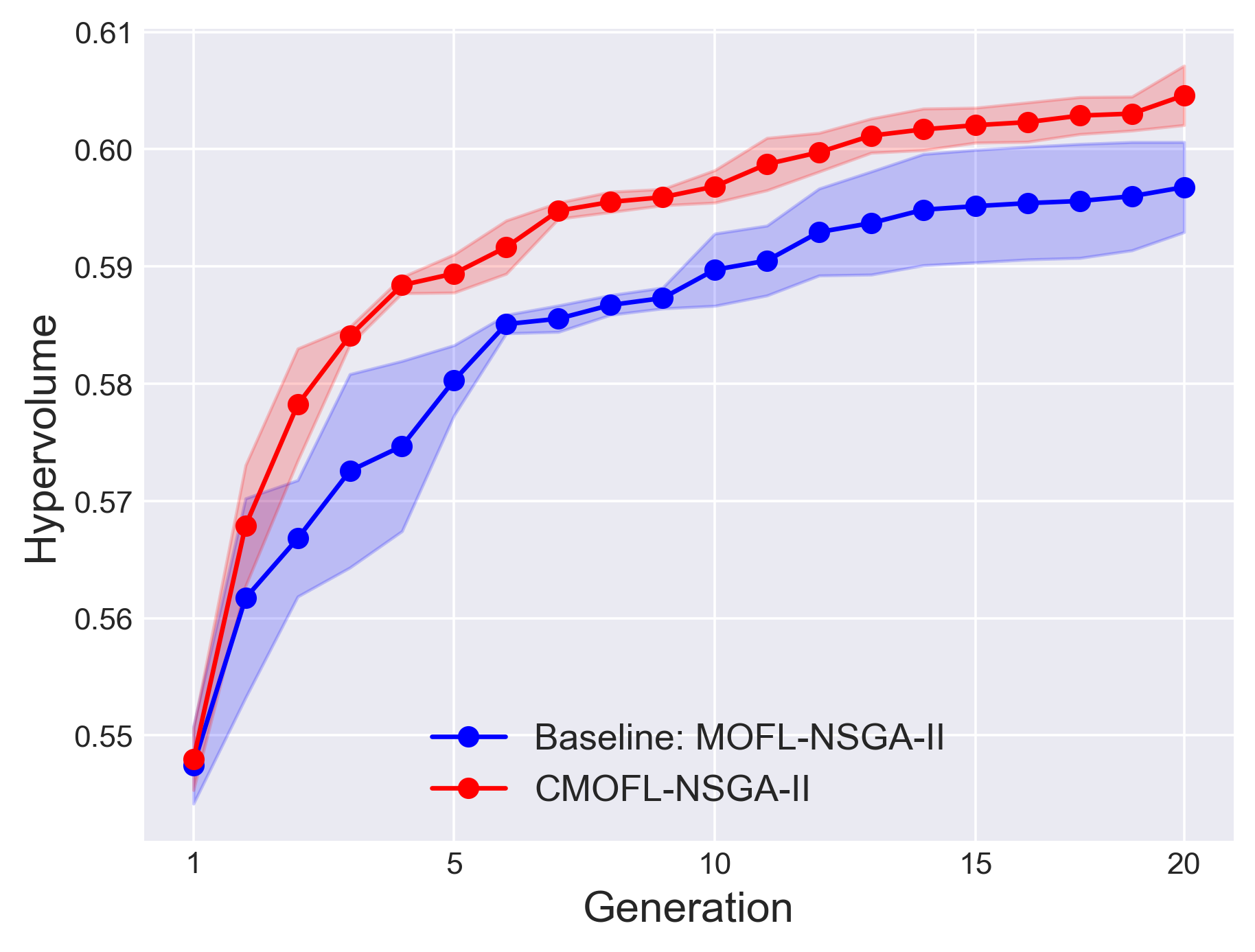}
         \subcaption{\scriptsize{SF: CMOFL-NSGA-II vs MOFL-NSGA-II}}
		\end{subfigure}

           \begin{subfigure}{0.32\textwidth}
  		 	\includegraphics[width=1\textwidth]{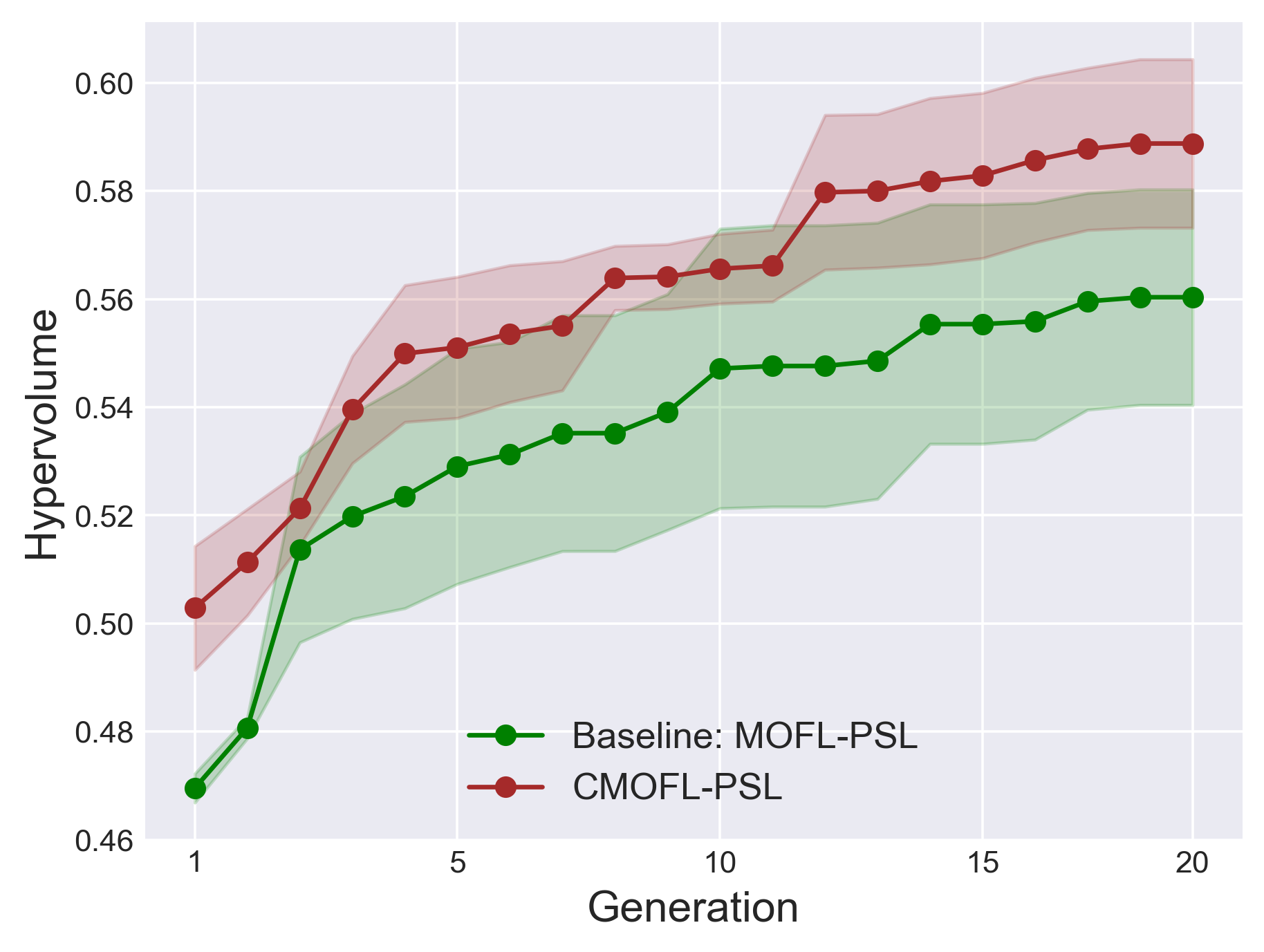}
      \subcaption{\scriptsize{BC: CMOFL-PSL vs MOFL-PSL}}
    		\end{subfigure}
    	\begin{subfigure}{0.32\textwidth}
  		 	\includegraphics[width=1\textwidth]{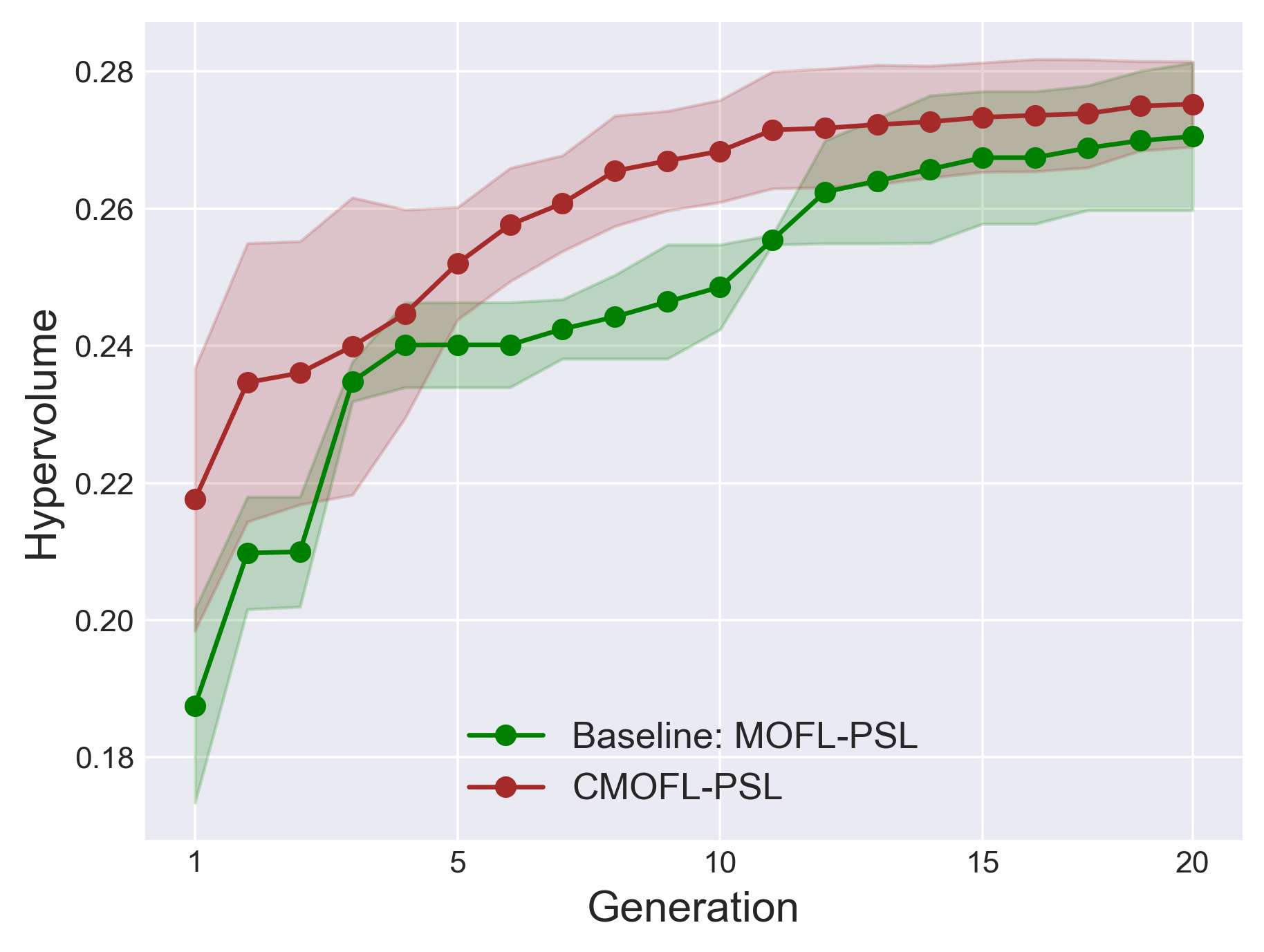}
            \subcaption{\scriptsize{RD: CMOFL-PSL vs MOFL-PSL}}
    		\end{subfigure}
   \begin{subfigure}{0.32\textwidth}
			\includegraphics[width=1\textwidth]{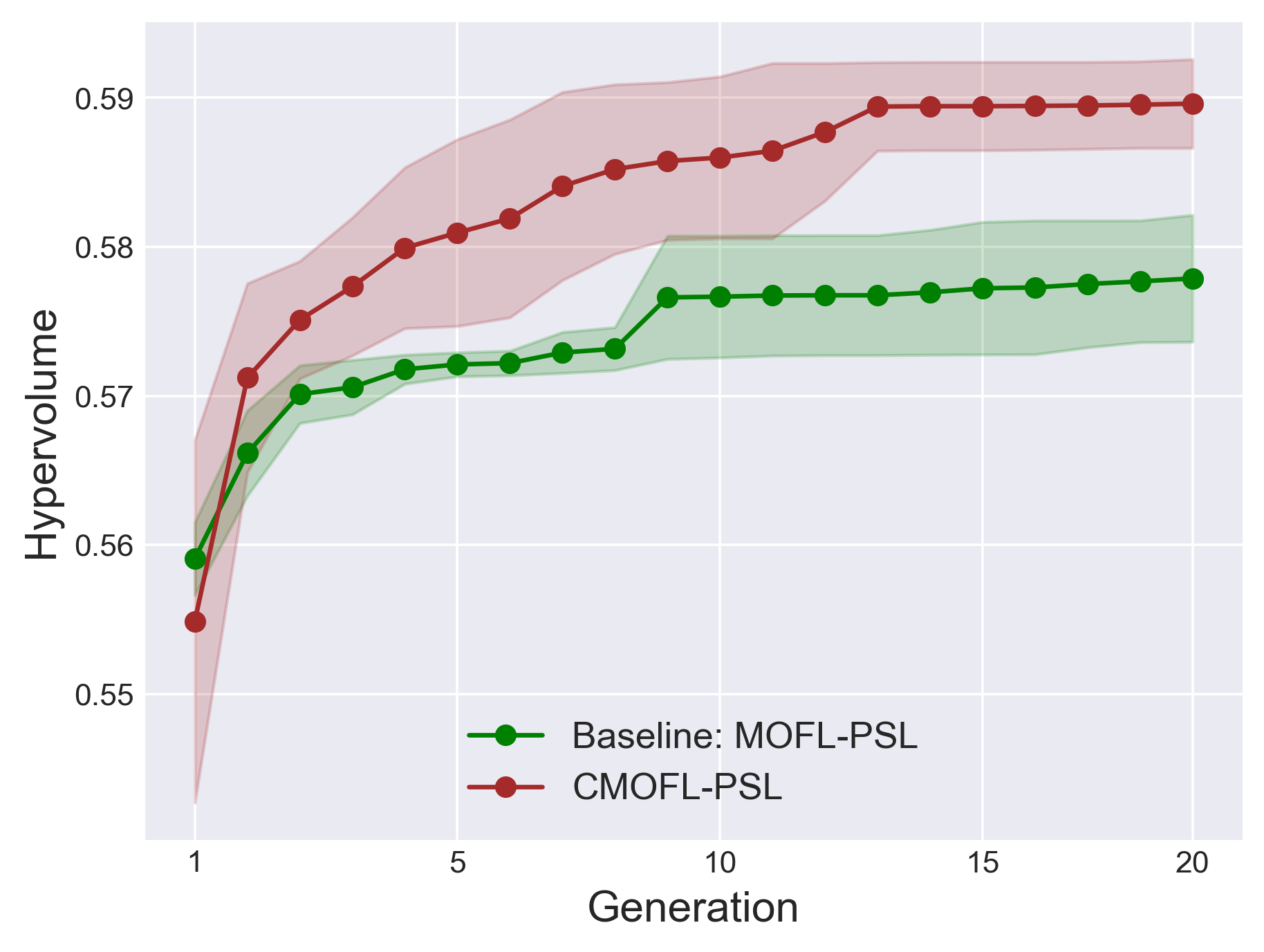}
         \subcaption{\scriptsize{SF: CMOFL-PSL vs MOFL-PSL}}
		\end{subfigure}
\vspace{-0.8em}
 \caption{Comparing hypervolume values of our proposed CMOFL algorithms with those of baseline MOFL algorithms on the CIFAR10 dataset for BC, RD, and SF, respectively. The first line compares CMOFL-NSGA-II and MOFL-NSGA-II for the three protection mechanisms. The second line compares CMOFL-PSL and MOFL-PSL for the three protection mechanisms.}
	\label{fig:hv_cifar}
 \vspace{-1em}
\end{figure*}

Figures \ref{fig:hv_fmnist} and \ref{fig:hv_cifar} illustrate the experimental results conducted on Fashion-MNIST and CIFAR10, respectively. The results are averaged over 3 different random seeds. They show that CMOFL-NSGA-II (red) and CMOFL-PSL (brown) achieve better hypervolume (HV) values than their corresponding baselines, MOFL-NSGA-II (blue) and MOFL-PSL (green), from the beginning to the 20th generation, demonstrating that CMOFL-NSGA-II and CMOFL-PSL, leveraging constraints to restrict the search for feasible solutions, can find better Pareto optimal solutions more efficiently.


    
\begin{figure*}[!h]
\vspace{-3pt}
	\centering
      	\begin{subfigure}{0.32\textwidth}
  		 	\includegraphics[width=1\textwidth]{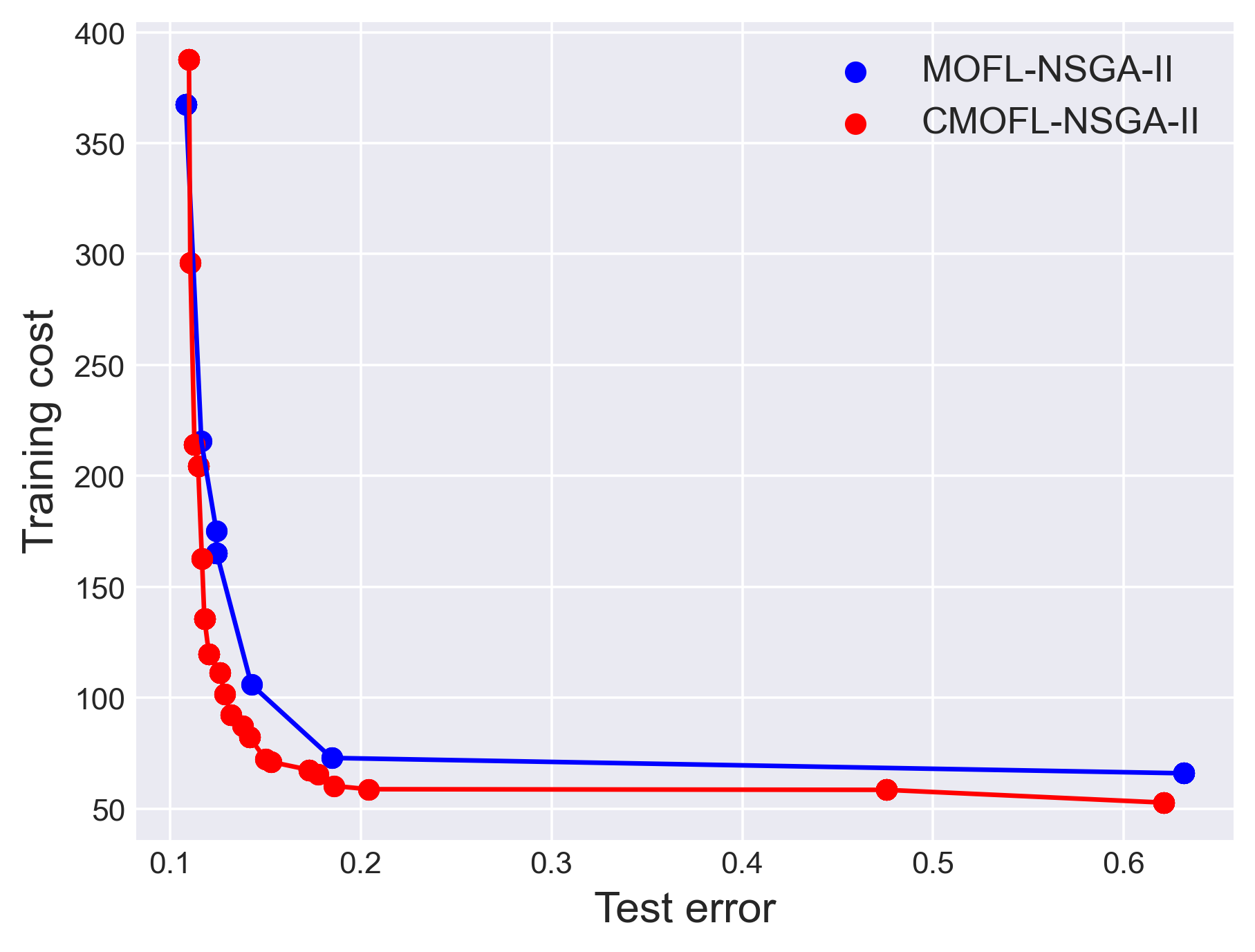}
      \subcaption{\scriptsize{BC: CMOFL-NSGA-II vs MOFL-NSGA-II}}
    		\end{subfigure}
    	\begin{subfigure}{0.32\textwidth}
  		 	\includegraphics[width=1\textwidth]{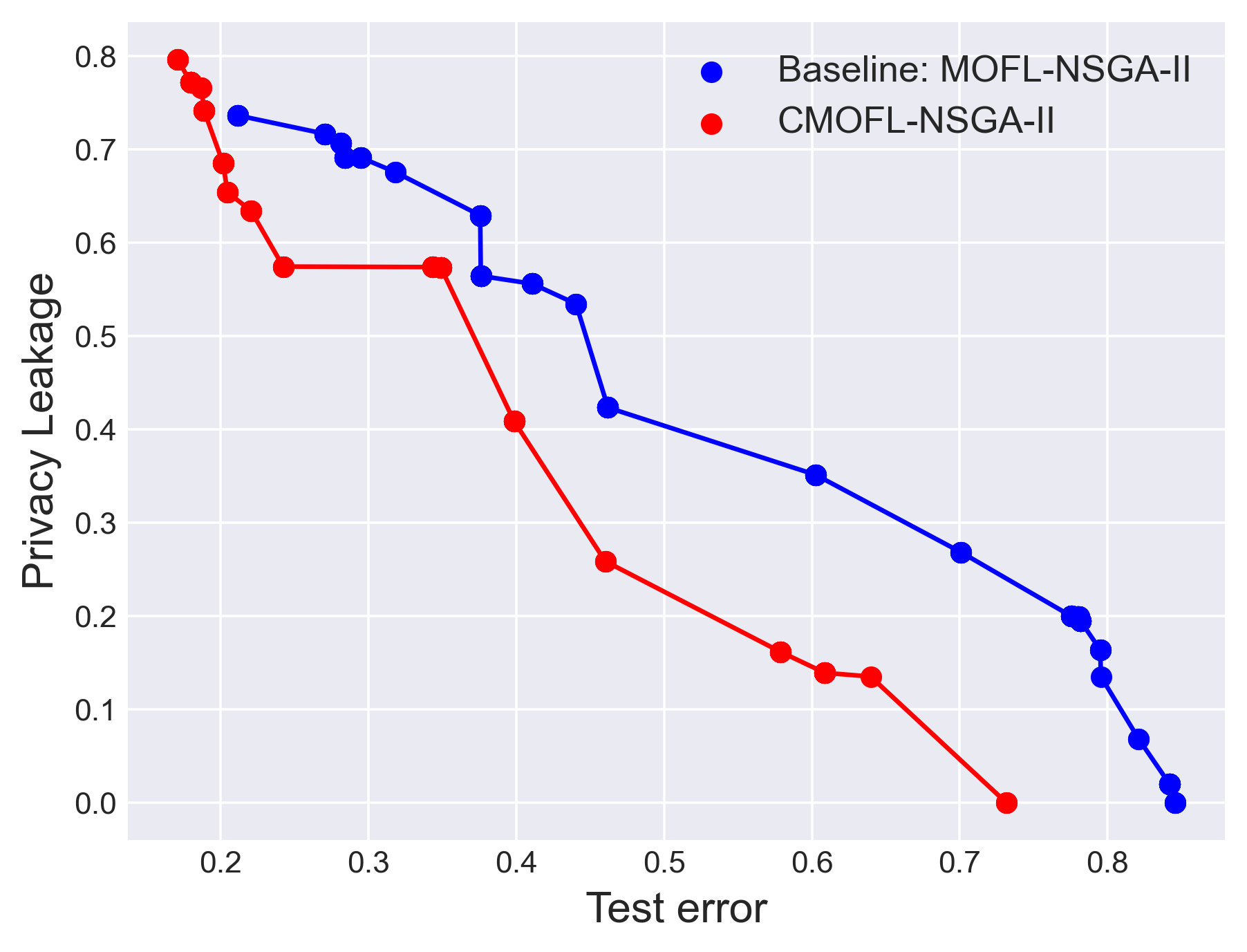}
            \subcaption{\scriptsize{RD: CMOFL-NSGA-II vs MOFL-NSGA-II}}
    		\end{subfigure}
   \begin{subfigure}{0.32\textwidth}
			\includegraphics[width=1\textwidth]{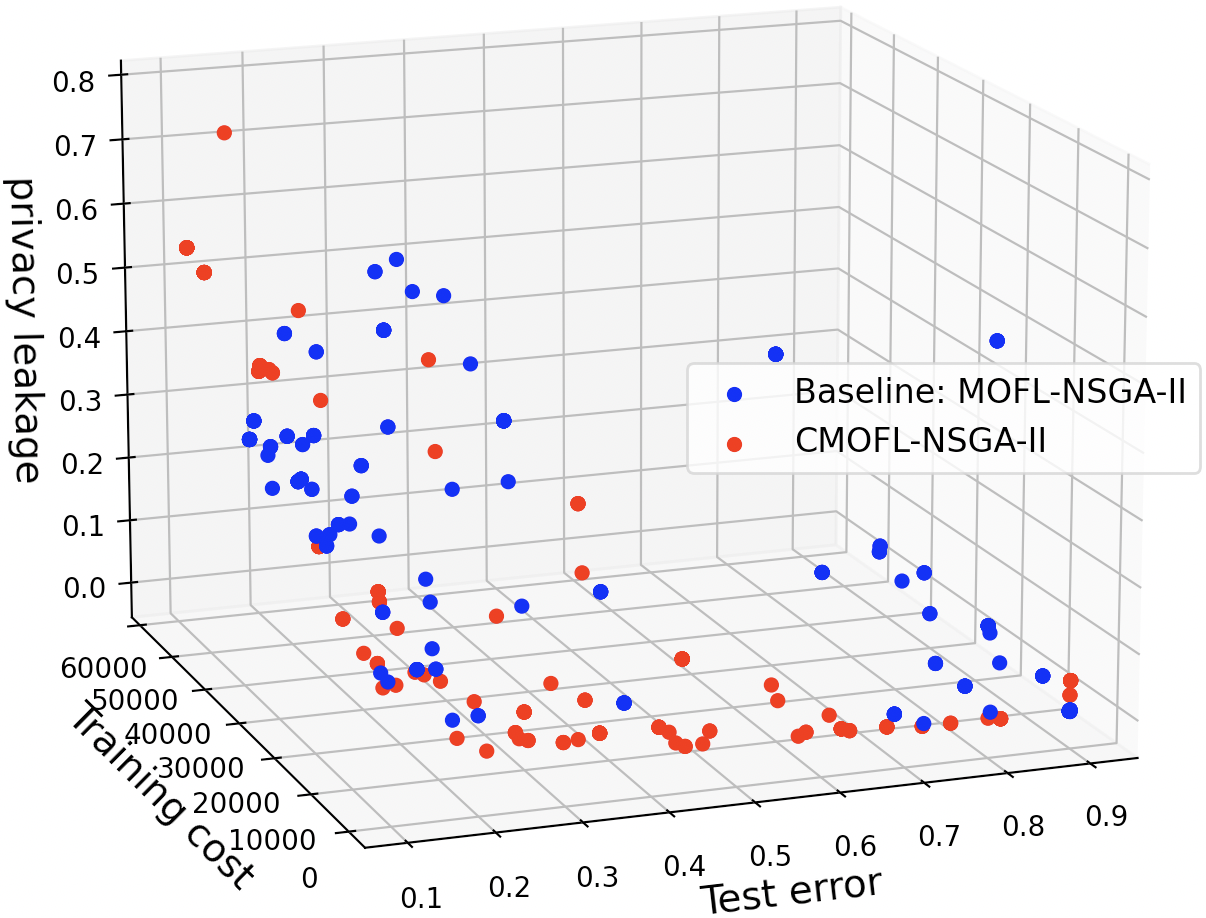}
         \subcaption{\scriptsize{SF: CMOFL-NSGA-II vs MOFL-NSGA-II}}
		\end{subfigure}

           \begin{subfigure}{0.32\textwidth}
  		 	\includegraphics[width=1\textwidth]{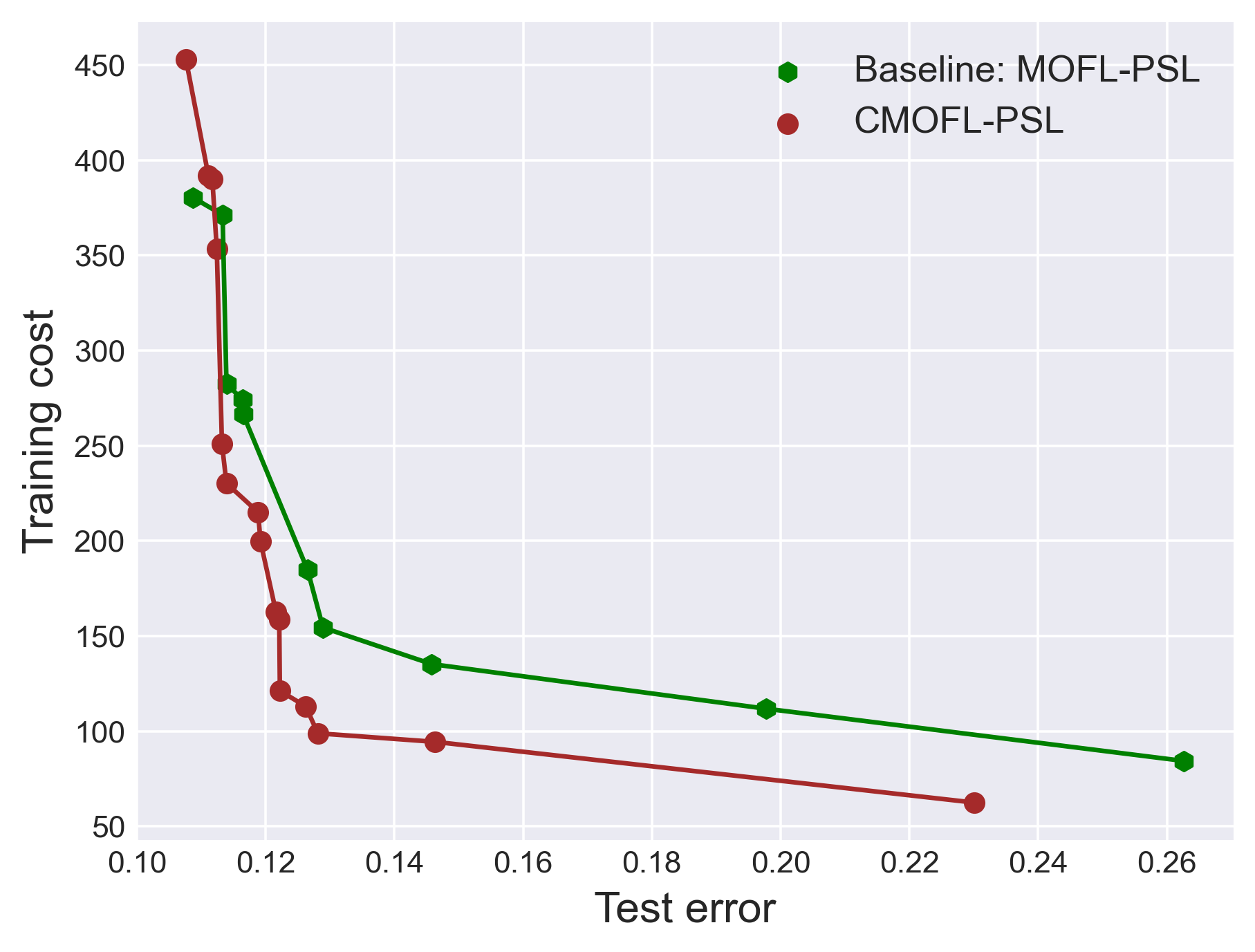}
      \subcaption{\scriptsize{BC: CMOFL-PSL vs MOFL-PSL}}
    		\end{subfigure}
    	\begin{subfigure}{0.32\textwidth}
  		 	\includegraphics[width=1\textwidth]{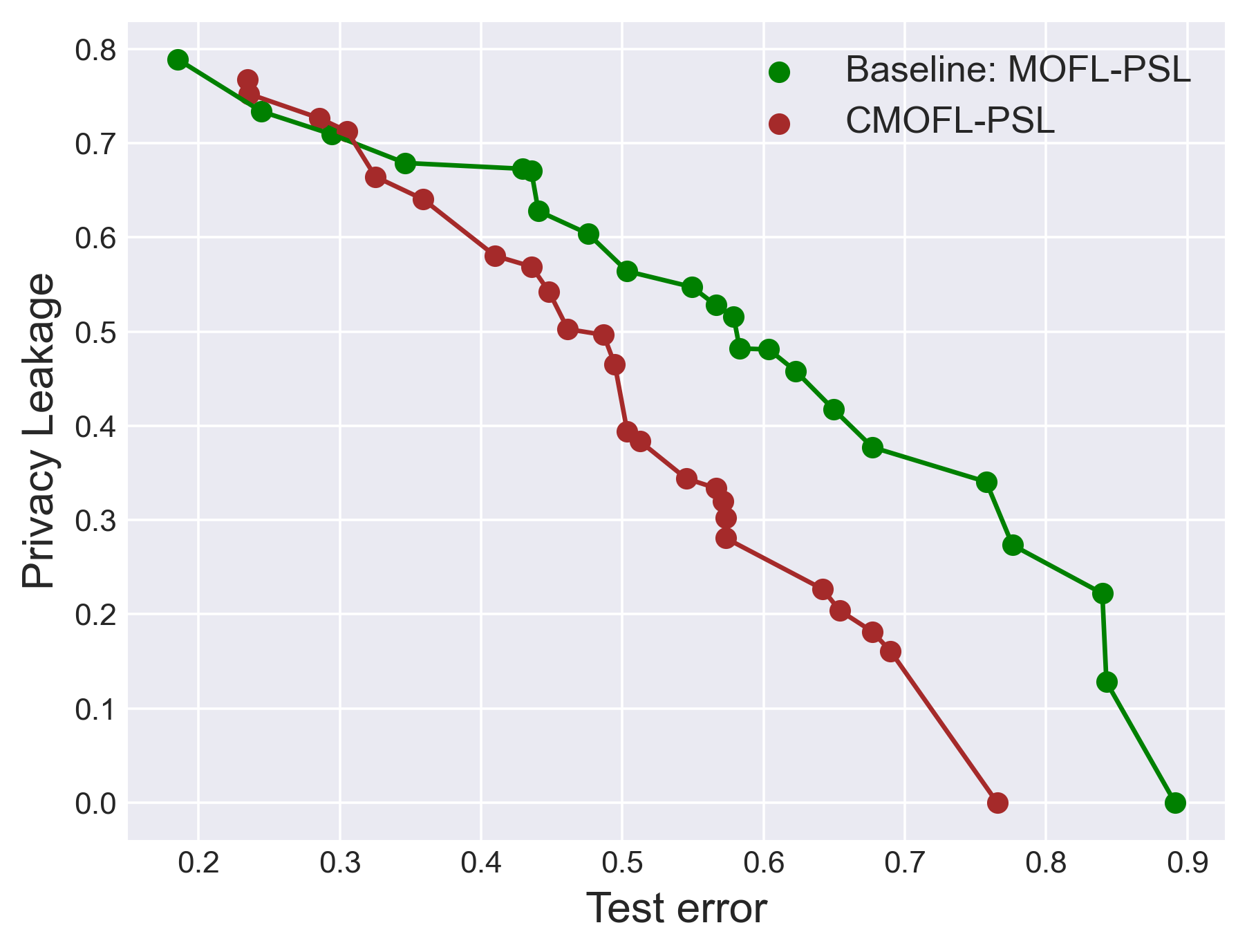}
            \subcaption{\scriptsize{RD: CMOFL-PSL vs MOFL-PSL}}
    		\end{subfigure}
   \begin{subfigure}{0.32\textwidth}
			\includegraphics[width=1\textwidth]{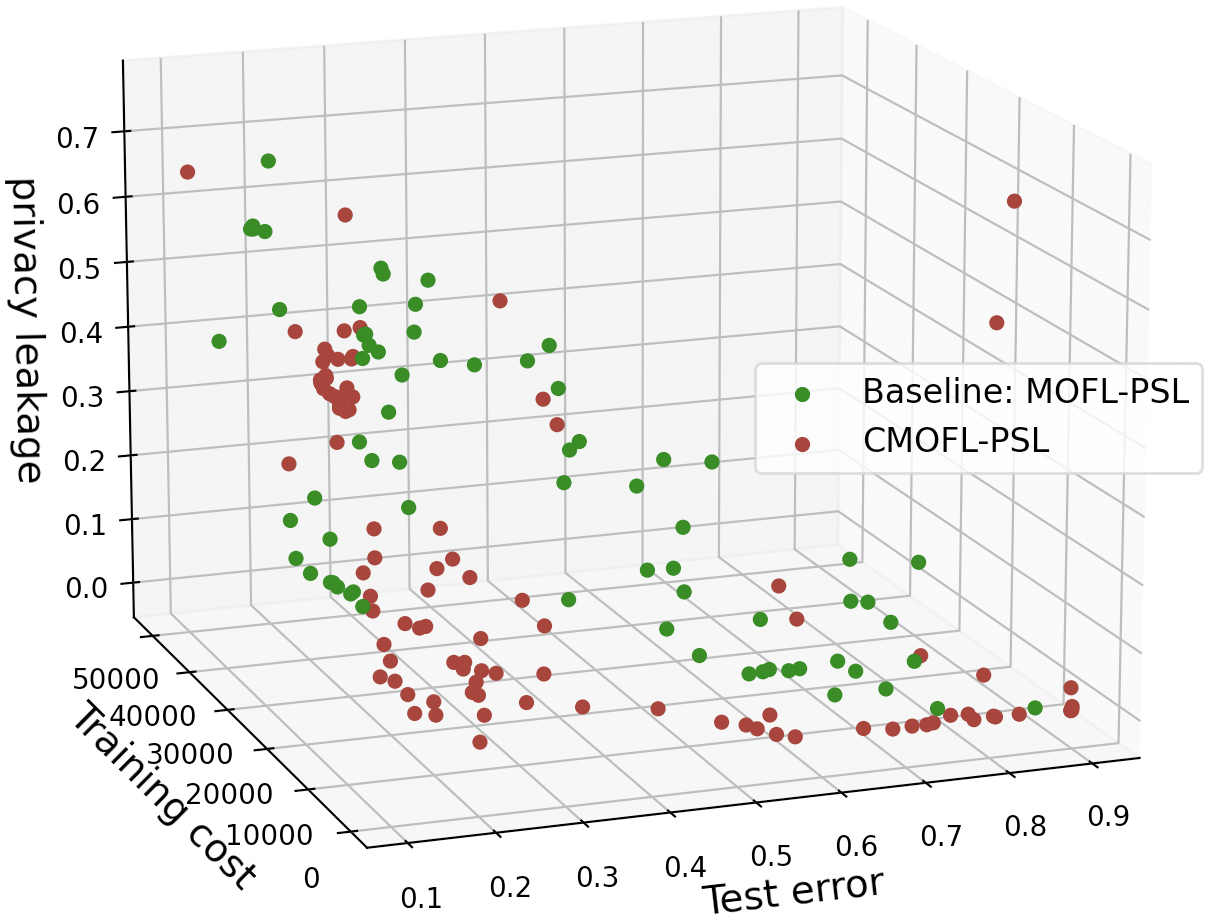}
         \subcaption{\scriptsize{SF: CMOFL-PSL vs MOFL-PSL}}
		\end{subfigure}

\vspace{-0.8em}
 \caption{Comparing Pareto fronts (at the 20th generation) of our proposed CMOFL algorithms and those of baseline MOFL algorithms on the Fashion-MNIST dataset for BC, RD, and SF, respectively. The first line shows CMOFL-NSGA-II (red) vs. MOFL-NSGA-II (blue) for each protection mechanism. The second line shows CMOFL-PSL (brown) vs. MOFL-PSL (green) for each protection mechanism. \textit{A better Pareto front curve should be more toward the bottom-left corner of each sub-figure}.}
	\label{fig:pareto_front}
 \vspace{-4pt}
\end{figure*}

We then compare the Pareto fronts (at the 20th generation) achieved by our proposed CMOFL algorithms with those achieved by baselines on the Fashion-MNIST dataset for BC, RD, and SF, respectively. Figure \ref{fig:pareto_front} illustrates the results (A better Pareto front curve should be more toward the bottom-left corner of each sub-figure.). It shows that our proposed CMOFL algorithms, CMOFL-NSGA-II (red) and CMOFL-PSL (brown), yield better Pareto fronts than baselines, MOFL-NSGA-II (blue) and MOFL-PSL (green), for all three protection mechanisms. 

\subsection{Effectiveness of CMOFL with a Limited FL Evaluation Budget}
CMOFL algorithms need to call the FL procedure to obtain real objective values for each solution. Therefore, if a CMOFL algorithm requires a large number of FL evaluations to find satisfactory Pareto optimal solutions, it is impractical in many real-world applications. 

To investigate the effectiveness of CMOFL-NSGA-II and CMOFL-PSL with limited FL evaluation budgets, we run each algorithm with 5 initial solutions and 20 generations that each has a population size of 5 (there are totally 105 FL evaluations, while experiments in Sec. \ref{main_results} involve 420 evaluations). We conduct experiments on the Fashion-MNIST dataset for BC, RD, and SF, respectively. The reported results are averaged over 3 different random seeds.

\begin{figure*}[!h]
\vspace{-3pt}
	\centering
           \begin{subfigure}{0.32\textwidth}
  		 	\includegraphics[width=1\textwidth]{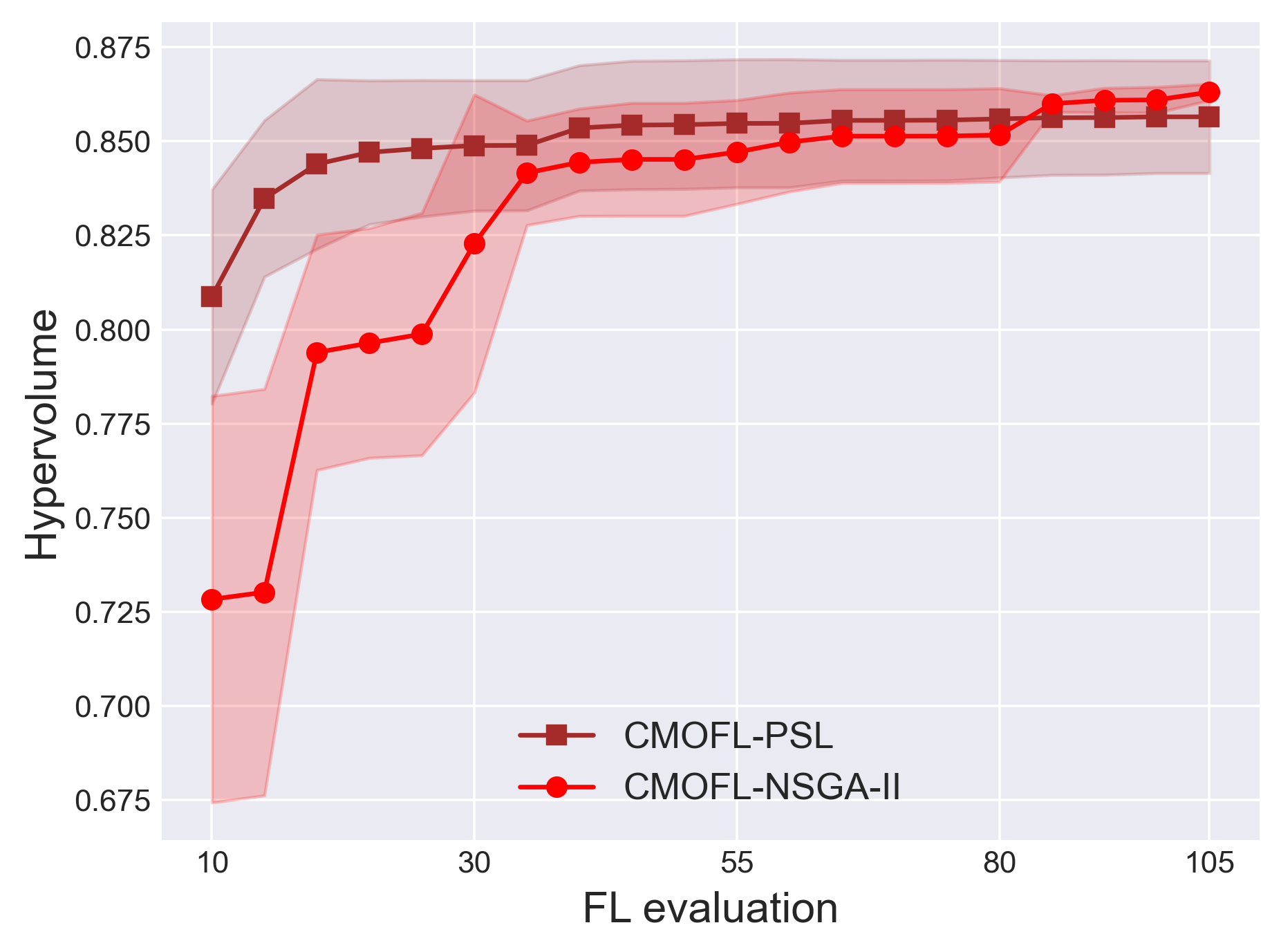}
      \subcaption{\scriptsize{BC: CMOFL-PSL vs CMOFL-NSGA-II}}
    		\end{subfigure}
    	\begin{subfigure}{0.32\textwidth}
  		 	\includegraphics[width=1\textwidth]{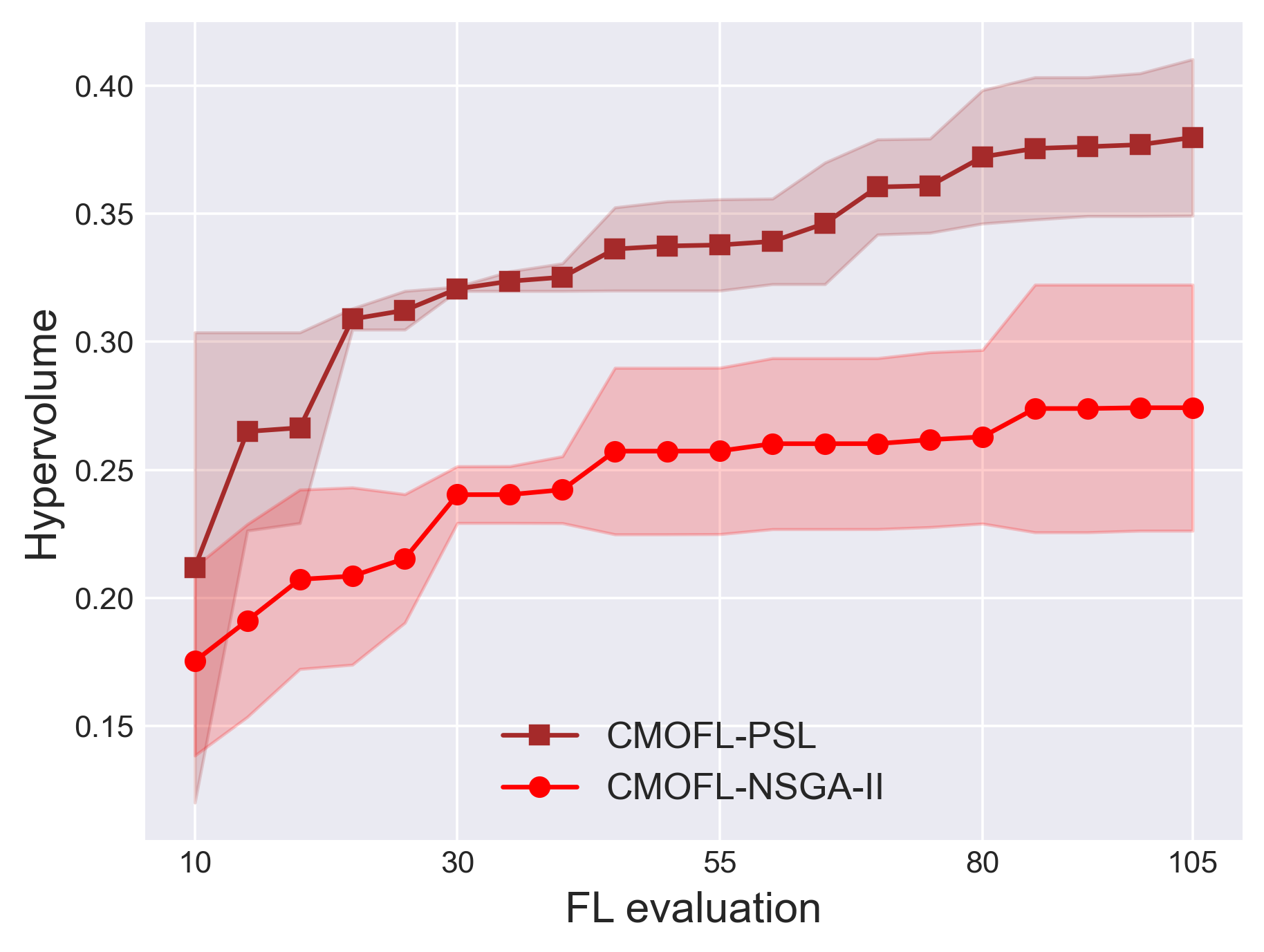}
            \subcaption{\scriptsize{RD: CMOFL-PSL vs CMOFL-NSGA-II}}
    		\end{subfigure}
   \begin{subfigure}{0.32\textwidth}
			\includegraphics[width=1\textwidth]{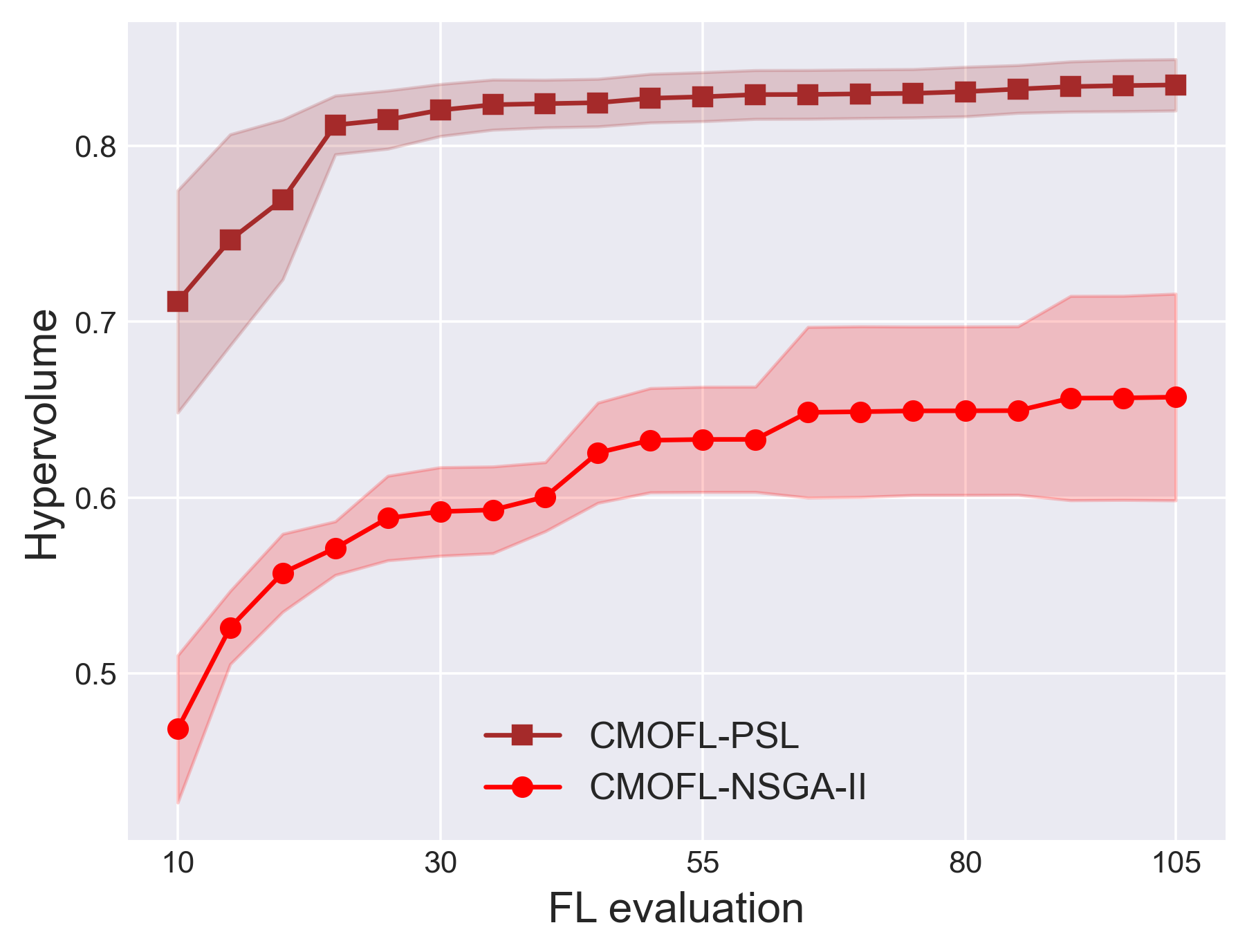}
         \subcaption{\scriptsize{SF: CMOFL-PSL vs CMOFL-NSGA-II}}
		\end{subfigure}

\vspace{-0.8em}
 \caption{Comparing hypervolume values of CMOFL-PSL (brown) and CMOFL-NSGA-II (red) when each generation uses 5 federated learning evaluations. Comparisons are conducted on the Fashion-MNIST dataset for BC, RD, and SF, respectively.}
	\label{fig:psl_nsga2}
 \vspace{-4pt}
\end{figure*}

Figure \ref{fig:psl_nsga2} compares CMOFL-PSL (brown) and CMOFL-NSGA-II (red) in terms of their hypervolume trends with each generation using 5 FL evaluations. It shows that CMOFL-PSL noticeably outperforms CMOFL-NSGA-II from the beginning of the optimization and converges to hypervolume values comparable to or better than CMOFL-NSGA-II, demonstrating that CMOFL-PSL can find better Pareto optimal solutions more efficiently than CMOFL-NSGA-II when the FL evaluation budget is limited. These results are expected because, per generation, CMOFL-PSL selects 5 best solutions from a large set of candidate solutions (1000 in this work) that are sampled from the Pareto set model and evaluated using a surrogate model (i.e., Gaussian Process) of the FL. In this way, CMOFL-PSL can efficiently explore broader solution space to find better Pareto optimal solutions. 

These experiments manifest that surrogate model-based MOO approaches, such as MOBO, are promising to solve TFL optimization problems in real-world scenarios where the FL evaluation budget is often limited. 

\subsection{Pareto Front as Guideline for Choosing Protection Mechanisms}
The Pareto front represents different optimal trade-offs among privacy leakage, utility loss, and training cost in this work. It helps define the applicable boundaries of protection mechanisms and hence could guide FL practitioners to choose the most appropriate protection mechanisms for applications at hand. For example, when Randomization is applied (see Figure \ref{fig:pareto_front} (b) and (e)), the Pareto front tells us that a noticeable amount of privacy leakage would be incurred as the test error decreases and vice versa. Therefore, it is challenging to find Randomization hyperparameters ($\sigma_{\text{rd}}$ and $c_{\text{clip}}$) that meet the requirements for applications in which both privacy and utility are critical (e.g., healthcare). Taking the BatchCrypt as another example (see Figure \ref{fig:pareto_front} (a) and (d)), the training cost increases significantly on the Pareto front as the utility loss approaches the optimal. This manifest that BatchCrypt is unsuitable for applications where efficiency and utility are of utmost importance (e.g., advertisement).

\section{Conclusion, Impact and Future Work}

In this paper, we formulate the problem of constrained multi-objective federated learning (CMOFL), which unifies multi-objective optimization and trustworthy federated learning. For one of the first attempts, we consider privacy leakage as an objective to minimize in CMOFL. We develop two improved CMOFL algorithms based on NSGA-II and PSL, respectively, for effectively and efficiently finding Pareto optimal solutions. Each algorithm leverages a regret function to penalize objectives that violate prespecified constraints. We design specific measurements of privacy leakage, utility loss, and training cost for three privacy protection mechanisms: differential privacy, BatchCrypt, and sparsification. Empirical experiments conducted under each of the three protection mechanisms demonstrate the efficacy of our proposed algorithms.

The benefits of finding Pareto optimal solutions (and front) for trustworthy federated learning (TFL) problems include: (1) Pareto optimal solutions can flexibly support participants' requirements. (2) Pareto front helps define applicable boundaries of privacy protection mechanisms. (3) Pareto front can serve as a tool to guide the standard-setting for privacy levels. 

At least three research directions are worth exploring: (1) Designing CMOFL algorithms that can find Pareto optimal solutions using a small amount of federated learning evaluation budget. In CMOFL, the evaluation of each solution is performed by federated learning, which is a time-consuming procedure. Therefore, a large number of federated learning evaluations is infeasible. (2) Applying CMOFL to vertical or hybrid FL settings. For example, we can leverage CMOFL to find Pareto optimal hyperparameters (e.g., tree number, tree depth, learning rate, leaf purity, and the number of aligned samples) that simultaneously optimize privacy, efficiency, and utility of SecureBoost~\cite{cheng2021secureboost,chen2021sbplus}, which is a widely used algorithm in vertical FL applications. (3) Investigating objectives to optimize beyond privacy, utility, and efficiency.

\newpage
\appendix
\setcounter{equation}{0}
\setcounter{theorem}{0}
\setcounter{prop}{0}
\setcounter{definition}{0}

\section{Privacy Leakage Measurements for Protection Mechanisms}\label{sec:applications}

In this section, we prove the correctness of privacy leakage measurements provided in Eq. (\ref{eq:dp_privacy_measure}) (for Randomization) and provided in  Eq. (\ref{equ:sf_privacy_measure}) (for Sparsification).



\subsection{Privacy Leakage}
Following \cite{du2012privacy,rassouli2019optimal}, the distortion (i.e., protection) extent is defined as the distance between the distribution $P_k^\calRO$ of original model parameters $W_k^\calRO$ (i.e., $W_k^\calRO \sim P_k^\calRO$) and the distribution $P_k^\calD$ of protected model parameters $W_k^\calD$ (i.e., $W_k^\calD \sim P_k^\calD$). In this work, we leverage the distortion extent to formulate the privacy leakage:
\begin{align}\label{eq: def_of_pl-app}
\epsilon_p = \frac{1}{K}\sum_{k=1}^K \epsilon_{p,k} \text{,  where } \epsilon_{p,k} = 1 - \text{TV}(P_k^\calRO||P_k^\calD),
\end{align}
where $\text{TV}(\cdot||\cdot)$ denotes the Total Variation distance between two distributions. A larger distortion applied to the original model parameters leads to larger $\text{TV}(P_k^\calRO||P_k^\calD)$, thereby less privacy leakage.

In this paper, suppose $W_k^\calRO$ to be the parameters sampled from the Multivariate Gaussian distribution $P_k^{\calRO} = \calN(\mu_0,\Sigma_0)$, where $\mu_0 = (\mu_u, \mu_o), \mu_u \in\mathbb R^q, \mu_o \in\mathbb R^{n-q}$ and $\Sigma_0 = \text{diag}(\Sigma_u^{q\times q}, \Sigma_o^{(n-q)\times (n-q)})$ is a diagonal matrix. Before we provide the proof for Proposition \ref{prop:1}, we first introduce the following two lemmas on estimating the total variation distance between two Gaussian distributions.
\begin{lemma}\label{lem:TV-gaussian}
(Total variation distance between Gaussians with the same mean \cite{devroye2018total}). Let $\mu \in R^n$, $\Sigma_1, \Sigma_2$ be diagonal matrix, and let $\lambda_1, \dots, \lambda_n$ denote the eigenvalues of $\Sigma_1^{-1}\Sigma_2-I_n$. Then, 
\begin{equation}
    \frac{1}{100} \leq \frac{{\text{TV}}(\calN(\mu, \Sigma_1), \calN(\mu, \Sigma_2))}{\min\{1,\sqrt{\sum_{i=1}^n\lambda_i^2}\}} \leq \frac{3}{2}.
\end{equation}
\end{lemma}

\begin{lemma} \label{lem:TV-diffmean-Gauss}
(Total variation distance between Gaussians with different means \cite{devroye2018total})
Assume that $\Sigma_1, \Sigma_2$
are positive definite, and let
\begin{equation}
    h = h(\mu_1, \Sigma_1, \mu_2, \Sigma_2) =\big{(}1-\frac{\text{det}(\Sigma_1)^{1/4}\text{det}(\Sigma_2)^{1/4}}{\text{det}(\frac{\Sigma_1+\Sigma_2}{2})^{1/2}}\exp\{-\frac{1}{8}(\mu_1-\mu_2)^T(\frac{\Sigma_1+\Sigma_2}{2})^{-1}(\mu_1-\mu_2)\}\big{)}^{1/2}
\end{equation}
Then, we have
\begin{equation}
    h^2 \leq {\text{TV}}(\calN(\mu_1, \Sigma_1), \calN(\mu_2, \Sigma_2)) \leq \sqrt{2}h.
\end{equation}
\end{lemma}

\begin{prop} \label{prop:1}
We have the following measurements of the privacy leakage based on the estimation of $\text{TV}(P_k^\calRO||P_k^\calD)$ for Randomization \cite{abadi2016deep} and Sparsification \cite{nori2021fast}:
\begin{itemize}
    \item \textbf{Randomization}: 
    \begin{equation}
     \epsilon_p = 1 - \min\{1, \Theta(\frac{\sigma_{\text{rd}}^2}{c_{\text{clip}}^2}\sqrt{d_w})\},
\end{equation}    
    where $\sigma_{\text{rd}}$ is the variance of Gaussian noise, $c_{\text{clip}}$ is the clip norm and $d_w$ is the dimension of model weights.
    \item \textbf{Sparsification}: \begin{equation}\label{equ:sf_privacy_measure_2}
    \epsilon_p = \frac{1}{K}\sum_{k=1}^K\left(1-\sqrt{2}(1-O(\exp\{- \mu_k \}))^{-1/2}\right),
    \end{equation}
    where $\mu_k$ is the mean of remained weights of $k_{th}$ client. 
\end{itemize}
\end{prop}

\begin{proof}
For \textbf{randomization mechanism}, the variance of each clients' model weights are first changed to $c_{\text{clip}}$, and then $W_k^\calD = W_k^\calRO + \calN(0, \sigma_{\text{rd}}^2)$. Therefore, $\text{TV}(P_k^\calRO, P_k^\calD) = \text{TV}(\calN(\mu_0, c_{\text{clip}}^2), \calN(\mu, c_{\text{clip}}^2+ \sigma_{\text{rd}}^2))$. According to Eq. \eqref{eq: def_of_pl-app}, we have
\begin{equation}\begin{split}
    \epsilon_{p,k} &= 1- \text{TV}(P_k^\calRO, P_k^\calD) \\
               &\geq 1 - \frac{3}{2} \sqrt{\sum_{i=1}^n\lambda_i^2} \\
               & = 1- \frac{3}{2}\sqrt{\sum_{i=1}^{d_w}(\frac{c_{\text{clip}}^2 + \sigma_{\text{rd}}^2}{c_{\text{clip}}^2}-1)^2} \\
               & = 1- \frac{3}{2}\sqrt{\sum_{i=1}^{d_w}\frac{\sigma_{\text{rd}}^4}{c_{\text{clip}}^4}} \\
               & = 1- \frac{3}{2}\frac{\sigma_{\text{rd}}^2}{c_{\text{clip}}^2}\sqrt{d_w} \\
               &  \geq 1- \min\{1,\frac{3}{2}\frac{\sigma_{\text{rd}}^2}{c_{\text{clip}}^2}\sqrt{d_w}\}
    \end{split}
\end{equation}
where the first inequality is due to Lemma \ref{lem:TV-gaussian} that $\text{TV}(P_k^\calRO, P_k^\calD) \leq \frac{3}{2} \sqrt{\sum_{i=1}^n\lambda_i^2} $, and the last inequality is because of TV distance does not exceed 1. On the other hand, 
\begin{equation}\begin{split}
    \epsilon_{p,k} &= 1- \text{TV}(P_k^\calRO, P_k^\calD) \\
               &\leq 1 - \frac{1}{100} \min\{1, \sqrt{\sum_{i=1}^n\lambda_i^2}\} \\
               & = 1- \min\{1,\frac{1}{100}\sqrt{\sum_{i=1}^{d_w}(\frac{c_{\text{clip}}^2 + \sigma_{\text{rd}}^2}{c_{\text{clip}}^2}-1)^2} \}\\
               & = 1- \frac{1}{100}\min\{1,\sqrt{\sum_{i=1}^{d_w}\frac{\sigma_{\text{rd}}^4}{c_{\text{clip}}^4}}\} \\
               & = 1- \frac{1}{100}\min\{1,\frac{\sigma_{\text{rd}}^2}{c_{\text{clip}}^2}\sqrt{d_w}\} \\
               & \leq 1- \min\{1,\frac{1}{100}\frac{\sigma_{\text{rd}}^2}{c_{\text{clip}}^2}\sqrt{d_w}\}.
    \end{split}
\end{equation}
Therefore, we prove $\epsilon_p = \sum_{k=1}^K(1 - \min\{1, \Theta(\frac{\sigma_{\text{rd}}^2}{c_{\text{clip}}^2}\sqrt{d_w})\})$ for randomization mechanism.

In the \textbf{sparsification} mechanism, clients initiate the sparsity mechanism by transmitting partial parameters to the server. Specifically, we assume that each client transmits the first $q$ dimensions to the server, without loss of generality. We further assume that the vector composed of the last $(n-q)$ dimensions follows a Gaussian distribution, denoted by $\calN(\mu_g, \Sigma_g)$, where $\Sigma_g$ is a diagonal matrix. In this context, the private information of the model follows $P_k^\calD \sim \calN(\mu, \Sigma)$, where $\mu = (\mu_u, \mu_g)$ and $\Sigma = \text{diag}(\Sigma_u^{q \times q}, \Sigma_g^{(n-q) \times (n-q)})$. 
According to Eq. \eqref{eq: def_of_pl-app} and Lemma \ref{lem:TV-diffmean-Gauss}, we obtain
\begin{equation}\label{eq:sparse-1}
    \epsilon_{p,k} \leq 1- \sqrt{2}h.
\end{equation}
Since the variance $\Sigma$ is bounded and assume the $\mu_g = 0$, we have
\begin{equation*}\label{eq:sparse-2}
\begin{split}
  h &= h(\mu_0, \Sigma_0, \mu, \Sigma) =\big{(}1-\frac{\text{det}(\Sigma_0)^{1/4}\text{det}(\Sigma)^{1/4}}{\text{det}(\frac{\Sigma_0+\Sigma}{2})^{1/2}}\exp\{-\frac{1}{8}(\mu-\mu_0)^T(\frac{\Sigma+\Sigma_0}{2})^{-1}(\mu-\mu_0)\}\big{)}^{1/2} \\
  & = \big{(}1-\exp\{-\Theta(\mu_g^T\mu_g)\}\big{)}^{1/2}
  \end{split}
\end{equation*}
Combining Eq. \eqref{eq:sparse-1} and \eqref{eq:sparse-2}, we can further obtain
\begin{equation*}
    \epsilon_{p,k} \leq 1-\sqrt{2}(1-\exp\{- O(\mu_k) \})^{1/2},
\end{equation*}
where $\mu_k$ is $\mu_g$ for $k_{th}$ client, i.e., the mean of remained weights of $k_{th}$ client. Therefore, we have 
\begin{equation*}
    \epsilon_{p} =\sum_{k=1}^K\epsilon_{p,k}\leq 1-\sqrt{2}(1-\exp\{- O(\mu_k) \})^{1/2}.
\end{equation*}
\end{proof}


\section{Proofs of Theorem 1 and Theorem 2}
In this section, we provide proofs of Theorem 1 and Theorem 2. We simplify these proofs in a centralized setting because the federated learning optimization utilized in CMOFL-NSGA-II and CMOFL-PSL serve as a black-box objective function. In addition, we combine the constraint with the objective function, which is reasonable according to the Lemma \ref{theorem_1_lemma}.
\begin{lemma} \label{theorem_1_lemma} \cite{luenberger1984linear}
Assume $\Omega$ is the convex region, $f$ and $g$ are convex functions.
Assume also that there is a point $x_1$ such that $g(x_1) < 0$. Then, if $x^*$ solves
\begin{equation*}
\begin{split}
        &\min f(x)  \\
        \text{subject to } & g(x) <0  \text{ and } x\in \Omega ,
\end{split}
\end{equation*}
then there is a $\alpha$ with $\alpha \geq 0$  such that $x^*$ solves the Lagrangian relaxation problem
\begin{equation*}
\begin{split}
    \min f(x) + \alpha g(x) \\
    \text{subject to } x\in \Omega.
    \end{split}
\end{equation*}
\end{lemma}

\subsection{Proof for Theorem 1}
\begin{lemma} \label{lem:lemm4}
The work \cite{zheng2022first} considers the following ONEMINMAX and LOTZ benchmarks in multi-objective problems. Let $d$ be the dimension of the solution space.
\begin{itemize}
    \item \textbf{LOTZ:} If the population size $N$ is at least $5(d+1)$,
then the expected runtime is $O(d^2
)$ iterations and $O(Nd^2)$
fitness evaluations.
\item \textbf{ONEMINMAX:} if the population size $N$ is at least $4(d + 1)$, then the
expected runtime is $O(d \text{log} d)$ iterations and $O(Nd \text{log} d)$
fitness evaluations.
\end{itemize}
\end{lemma}
Lemma \ref{lem:lemm4} demonstrates the NSGA-II algorithm could obtain almost Pareto optimal solutions (i.e., within a small $\epsilon$ error) for LOTZ and ONEMINMAX benchmarks with sufficiently large population size $N$. 
Further, we provide the convergence analysis of Algo. \ref{alg:nsga_fl} when the objective values obtained by Algo. \ref{alg:nsga_fl} approaches the finite Pareto optimal objective values within $\epsilon$ error from the perspective of hypervolume.

\begin{theorem}
Let $Y^*$ be the optimal objective values w.r.t $m$ objectives. If for any $y^* \in Y^*, \exists y^T \in Y_T$ obtained by Algo. \ref{alg:nsga_fl}, s.t.  $\|y^T-y^*\| \leq \epsilon$, then we have
\begin{equation}
    \text{HV}_z(Y^*)-\text{HV}_z(Y_T) \leq Cm\epsilon,
\end{equation}
where $\text{HV}_z(\cdot)$ represents the hypervolume with reference point $z$ and $C$ is a constant. 
\end{theorem}

\begin{proof}
Let $z = [z_1, \cdots, z_m]$ be the reference point. If for any $y^*= [y^*_1, \cdots, y^*_m] \in Y^*, \exists y^T = [y^T_1, \cdots, y^T_m] \in Y_T$ obtained by Algo. \ref{alg:nsga_fl}, s.t.  $\|y^T-y^*\| \leq \epsilon$, we obtain
\begin{equation} \label{eq:refer-app}
\begin{split}
     \|z_i - y_i^*\| - \|z_i - y_i^T\|
     = \|y_i^T- y_i^*\| \leq \|y^T- y^*\| \leq \epsilon,
    \end{split}
\end{equation}
without loss of generality, assume $z,y \in (0,1)^m$, Eq. \eqref{eq:refer-app} leads to:
\begin{equation*}
\begin{split}
        \text{Volume}(y^*, r) - \text{Volume}(y^T, r) & = \prod_{i=1}^m(z_i-y_i^*)- \prod_{i=1}^m(z_i-y_i^T)\\
        & \leq \prod_{i=1}^m(z_i-y_i^T+\epsilon) - \prod_{i=1}^m(z_i-y_i^T)\\
        & = \epsilon\sum_{k=0}^{m-1}\sum_{i_1<i_2<\cdots<i_k}(z_{i_1}-y_{i_1})\cdots(z_{i_k}-y_{i_k})(\epsilon)^{m-k-1} \\
        & \leq \sum_{k=0}^{m-1}\tbinom{m}{k}(\epsilon)^{m-k} \\
        & = (1+\epsilon)^m - 1
        \end{split}
\end{equation*}
In addition, as $\epsilon$ tends to zero, the $(1+\epsilon)^m - 1 $ tends to $m\epsilon$. Since the $Y^*$ is finite set, then
\begin{equation*}
    \text{HV}_z(Y^*)-\text{HV}_z(Y_T) \leq Cm \epsilon,
\end{equation*}
where $C$ is a constant.
\end{proof}
\subsection{Proof for Theorem 2}
In this section, we analyze the convergence of Algo. \ref{alg:psl_fl} through regret bound following works~\cite{zhang2020random,paria2020flexible}. 

Multi-objective Bayesian Optimization (MOBO) is an extension of single-objective Bayesian optimization that aims at solving expensive multi-objective optimization problems. The MOBO approach proposed by Paria et al. \cite{paria2020flexible} utilizes scalarization-based algorithms such as Tchebyshev scalarizations \cite{nakayama2009sequential} and Hypervolume scalarizations \cite{zhang2020random}. In this approach, the multi-objective problem is iteratively scalarized into single-objective problems with random preferences, followed by applying single-objective Bayesian Optimization to solve them. \\
\textbf{Tchebyshev scalarizations \cite{nakayama2009sequential}} is defined as:
\begin{equation}
    s_{\lambda}(y) = \max_{1\leq i \leq m}\lambda_i(y_i-z_i),
\end{equation}
where $z = [z_1, \cdots, z_m]$ is reference point, $y = [y_1, \cdots, y_m]$ and $\lambda = [\lambda_1, \cdots, \lambda_m]$. \\
\textbf{Bayes Regret.}
MOBO aims to return a set of points from the
user specified region $X\subset \calX$. This can be achieved by minimizing the following Bayes regret denoted by $\calR_B$,
\begin{equation*}
    \calR_B(X) = \EE_{\lambda \in p(\lambda)}\big(\max_{x \in \calX}s_{\lambda}(f(x))- \max_{x \in X}s_{\lambda}(f(x)\big),
\end{equation*}
\begin{equation*}
    X^*= \argmin_{X \subset\calX}\calR_B(X),
\end{equation*}
where $p(\lambda)$ is a prior. Based on the Bayes regret, the instantaneous regret incurred by our algorithm at step $t$ is defined as:
\begin{equation}
    r(x_t, \lambda^t) = \max_{x \in \calX}s_{\lambda^t}(f(x)) - s_{\lambda^t}(f(x_t)),
\end{equation}
where $x_t$ is obtained by Algo. 3. The
cumulative regret till step $T$ is defined as,
\begin{equation*}
    \calR_C(T) = \sum_{t=1}^T r(x_t, \lambda^t)
\end{equation*}

\begin{lemma}
\cite{paria2020flexible}\label{theorem2-app}
The expected regret of algorithm \ref{alg:psl_fl} after $T$ observations (i.e., generations) can be upper bounded:
\begin{equation}
    \EE\calR_C(T) \leq  m(\gamma_TlnT)^{1/2},
\end{equation}
where $\gamma_T$ is a kernel-dependent quantity known as the maximum information gain \cite{paria2020flexible}. For example,
$\gamma_T = O(\text{poly }ln(T))$ for the squared exponential kernel. The expectation $\EE$ is the expectation is taken over $\{\lambda^t\}_{t=1}^T$.
\end{lemma} 
Under the further assumption of the space of $\lambda$ being a bounded subset of a normed linear space, and the scalarizations $s_\lambda$ being Lipschitz in $\lambda$, it can be shown that $\EE r(x_t, \lambda^t) \leq \EE R_C(T) + o(1)$, which combined with Lemma \ref{theorem2-app} shows that the Bayes regret converges to zero as $T \to \infty$ \cite{paria2020flexible}.\\
\textbf{Relation between Bayes regret and Hypervolume regret.}
Hypervolume defined in Def. \ref{def:hypervolume} reflects the quality of a multi-objective algorithm. Zhang et al. formulates \textbf{Hypervolume scalarizations \cite{zhang2020random}} as:
\begin{equation}
    s_{\lambda}(y) = \max_{1\leq i \leq m}(\min(0, \frac{y_i-z_i}{\lambda_i})),
\end{equation}
Tchebyshev scalarizations are similar to hypervolume scalarizations, but they differ in that Tchebyshev scalarizations involve the multiplication of the coefficient $\lambda$, while hypervolume scalarizations involve the overriding of $\lambda$. The following Lemma \ref{lem:hyper-sca} formulates the hypervolume based on the hypervolume scalarizations:
\begin{lemma}\cite{deng2019approximating} \label{lem:hyper-sca}
Let $y =\{y_1, ..., y_m\}$ and $z =\{z_1, ..., z_m\}$ be two sets of $m$ points. Then, the hypervolume of $y$ with respect to a reference point $z$ is given
by:
\begin{equation}
   \text{HV}_{z}(y) = c_k\EE_{\lambda \in \calS_+^{k-1}}\max_{y\in Y}s_{\lambda}(y-z).
\end{equation}
\end{lemma}
Further, we obtain the Theorem 2 when the type of scalarizations method is hypervolume scalarization, 
\begin{theorem} \cite{zhang2020random} If the type of scalarizations method is hypervolume scalarization, then the hypervolume regret after $T$ observations is upper bounded as:
\begin{equation}
  \sum_{t=1}^T\big(\text{HV}_z(Y^*)- \text{HV}_z(Y_t)\big)\leq O(m^2d^{1/2}[\gamma_Tln(T)T]^{1/2}).
\end{equation}  
Furthermore,  $\text{HV}_z(Y^*)- \text{HV}_z(Y_T) \leq O(m^2d^{1/2}[\gamma_Tln(T)/T]^{1/2})$
\end{theorem}

\newpage
\bibliographystyle{ACM-Reference-Format}
\bibliography{reference}

\end{document}